\documentclass[10pt,]{article} 
\usepackage[numbers]{natbib}
\usepackage[english]{babel}
\usepackage{paralist}
\usepackage{amsmath,amssymb,amsthm,amsfonts,color,bbm}
\usepackage{thmtools,thm-restate}
\usepackage[left=1in,right=1in,top=1in,bottom=1in]{geometry}
\usepackage{tikz}
\usetikzlibrary{shapes}
\usepackage{float}
\usepackage{url}

\newtheorem{definition}{Definition}
\newtheorem{theorem}{Theorem}
\newtheorem{lemma}{Lemma}

\newtheorem{proposition}{Proposition}

\renewcommand{\c}{\mathcal}
\renewcommand{\b}{\mathbb}
\renewcommand{\P}{\text{\normalfont P}}
\newcommand{\bP}{\mathbb{P}}

\newcommand{\argmax}[1]{\underset{#1}{\text{argmax}}}
\newcommand{\argmin}[1]{\underset{#1}{\text{argmin}}}
\DeclareMathOperator*{\altargmin}{arg\,min}

\renewcommand{\hat}{\widehat}
\renewcommand{\tilde}{\widetilde}
\renewcommand{\nu}{\vartheta}
\newcommand{\abs}[1]{\left| #1 \right|}
\newcommand{\sset}[1]{\left\{#1\right\}}

\newcommand{\Da}{\sqrt{\frac{VC(\c{H})+\log{1/\delta}}{n}}}
\newcommand{\epsa}{\max_{ya}\sqrt{\frac{VC(\c{H})+\log{1/\delta}}{n\P_{ya}}}}

\newcommand{\Db}{\max_{ya}\sqrt{\frac{VC(\c{H})+\log{1/{\delta}}}{n\P_{ya}}}}
\newcommand{\epsb}{\sqrt{\frac{\log{1/\delta}}{n\P_{ya}}}}
\newcommand{\Yfair}{Y^*}

\newcommand\independent{\protect\mathpalette{\protect\independenT}{\perp}}
\def\independenT#1#2{\mathrel{\rlap{$#1#2$}\mkern2mu{#1#2}}}

\newcommand{\removed}[1]{}

\title{\vspace{-15mm}\rule{6.5in}{2pt}\\Learning Non-Discriminatory Predictors\\\rule[2mm]{6.5in}{0.5pt}}
\author{}
\date{}
\begin{document}
\maketitle
\vspace{-21mm}
\noindent
\textbf{Blake Woodworth} \hfill \url{blake@ttic.edu}\\
\textbf{Suriya Gunasekar}\hfill \url{suirya@ttic.edu}\\
\textbf{Mesrob I. Ohannessian}\hfill \url{mesrob@ttic.edu}\\
\textbf{Nathan Srebro} \hfill \url{nati@ttic.edu}\\
{\small Toyota Technological Institute at Chicago, Chicago, IL 60637, USA}

\vspace{5mm}

\begin{abstract}
  We consider learning a predictor which is non-discriminatory with
  respect to a ``protected attribute'' according to the notion of
  ``equalized odds'' proposed by  \cite{hardt2016equality}.  We study 
  the problem of learning such a
  non-discriminatory predictor from a finite training set, both
  statistically and computationally.  We show that a post-hoc
  correction approach, as suggested by Hardt et al, can be highly
  suboptimal, present a nearly-optimal statistical procedure,
  argue that the associated computational problem is intractable, and suggest
  a second moment relaxation of the non-discrimination definition for
  which learning is tractable.
\end{abstract}

\section{Introduction}
Machine learning algorithms are increasingly deployed in important
decision making tasks that affect people's lives significantly. These
tools already appear in domains such as lending, policing,
criminal sentencing, and targeted service offerings.  In many of these
domains, it is morally and legally undesirable to
discriminate based on certain ``protected attributes'' such as race
and gender.  Even in seemingly innocent applications, such as ad
placement and product recommendations, such discrimination might be
illegal or detrimental.  Consequently, there has been abundant public, academic
and technical interest in notions of non-discrimination and fairness,
and achieving ``equal opportunity by design'' is a major United
States national Big Data challenge, \cite{WhiteHouse16}.


We consider non-discrimination in supervised learning where
the goal is to learn a (potentially randomized) predictor $h(X)$
or\footnote{See \cite{hardt2016equality} for a discussion on why it
  might be necessary for a non-discriminatory predictor to use $A$
 } $h(X,A)$ for a target quantity $Y$ using features
$X$ and a protected attribute $A$, while
ensuring non-discrimination with respect to $A$. As an illustrative
example, consider a financial institution that wants to predict
whether a particular individual will pay back a loan or not,
corresponding to $Y = 1$ and $Y = 0$, respectively.  The features $X$ could include
financial as well as other information, e.g.~about education, driving,
and housing history, languages spoken, and the number of members
in the household, all of which have a potential of being
used inappropriately as a surrogate for a protected attribute $A$,
such as gender or race.  It is important that the predictor for loan
repayment not be even implicitly discriminatory with respect to $A$.


Recent work has addressed the issue of defining what it means to be non-discriminatory---both in the context of supervised learning e.g.~\cite{dwork2012fairness,pedreshi2008discrimination,feldman2015certifying}, and otherwise e.g.~\cite{joseph2016fairness,joseph2016rawlsian}. The particular notion of non-discrimination we consider here is
``equalized odds'', recently presented and studied by
\citet{hardt2016equality}:
\begin{definition}[Equalized odds] \label{def:non-discrimination}
A possibly randomized predictor $\hat{Y}\!\! = h(X,A)$ for target  
$Y$ is non-discriminatory with respect to a protected attribute $A$ if 
$\hat{Y}$ is independent of $A$ conditioned on $Y$.
\end{definition}
Informally, we require that even when the correct label $Y$ provides
information about the protected attribute $A$, if we already know $Y$,
the prediction $\hat{Y}$ does not provide any {\em additional} information about
$A$.  The definition can also be motivated in terms incentive
structure and of moving the burden of uncertainty from the protected
population to the decision maker.  See \citet{hardt2016equality} for
further discussion of the definition, its implications, and comparisons
to alternative notions.

In a binary prediction task with binary protected attribute, i.e.  $\hat{Y},A,Y\in\{0,1\}$, Definition
\ref{def:non-discrimination} can be qualified in terms of true and
false positive rates. Denote the group-conditional true and false
positive rates as,
\begin{equation} \label{eq:gammapop}
\gamma_{ya}(\hat{Y}) := \mathbb{P}( \hat{Y} = 1\ |\ Y = y, A = a ),
\end{equation}
Then Definition \ref{def:non-discrimination} is equivalent to
requiring that the class conditional true and false positive rates
agree across different groups (different values of $A$):
\begin{equation}
\gamma_{00}(\hat{Y}) = \gamma_{01}(\hat{Y}) \qquad \textrm{and}\qquad \gamma_{10}(\hat{Y}) = \gamma_{11}(\hat{Y})
\end{equation}
Returning to the loan example, this definition requires that the
percentage of men who are wrongly denied loans even though they would
have paid it back must match the corresponding percentage for
women, 
and similarly the percentage of men who are wrongly given loans that
they will not pay back must also match the corresponding percentage of women. 
This does
\emph{not} however require that the same percentage of male and female
applicants will receive loans. For instance, if women pay back loans
with \textit{truly} higher frequency than men, then the predictor
would be allowed to deny loans to men more often than women.

While \citeauthor{hardt2016equality}~focused on the notion itself and how it
behaves on the population, in this work we tackle the problem of how to learn
a good non-discriminatory predictor (i.e.~satisfying the equalized
odds) from a finite training
set.  We examine this both from a statistical perspective of how to
best obtain a predictor from finite data that would be as accurate and
non-discriminatory as possible on the population, and from a
computational perspective.  

One possible approach to learning a non-discriminative predictor is the
{\em post hoc correction} proposed by \citet{hardt2016equality}: first learn a
good, possibly discriminatory predictor. Afterwards, this predictor is ``corrected'' by taking into account $A$
 in order to make the predictor non-discriminatory.  When
$Y$ is binary and the predictors $\hat{Y}$ are real-valued, they show that the unconstrained Bayes
optimal least-square regressor can be post hoc corrected to the
optimal predictor with respect to the 0-1 loss.  In Section
\ref{sec:suboptimalityposthoc}, we consider more carefully the
limitations of such a post hoc procedure. In particular, we show that
this approach can fail for the 0-1 and hinge losses, even if the Bayes
optimal predictor with respect to those losses is learned in the
first step.  We also show that even when minimizing the squared loss,
the approach can fail once the hypothesis
class is constrained, as is essential when learning from finite data.  From this, we
conclude that post hoc correction is not sufficient, and that it is
necessary to directly incorporate non-discrimination into the learning
process.

Turning to learning from finite data, we cannot hope to ensure exact
non-discrimination on the population.  To this end, in Section
\ref{sec:detection} we define a notion of approximate
non-discrimination, motivate it, and explore its limits by
analyzing the statistical problem of detecting whether or not a
predictor is at least $\alpha$-discriminatory.

We then turn to the main statistical question: given a finite training
set, how can we best learn a predictor that is ensured to be as
non-discriminatory as possible (on the population) and competes (in
terms of its population loss) with the best non-discriminatory
predictor in some given hypothesis class (this is essentially an
extension of the notion of agnostic PAC learning with a
non-discrimination constraint).  In Section \ref{sec:binary} we show
that an ERM-type procedure, minimizing the training error subject
to an empirical non-discrimination constraint, is statistically
sub-optimal, and instead we present a statistically optimal (up to
constant factors) two-step learning procedure for non-discriminatory
binary classification.

Unfortunately, learning a non-discriminatory binary classifier is
computationally hard, which we prove in Section \ref{sec:hardness}.
In order to allow tractable training, in Section \ref{sec:2ndorder},
we present a relaxation of equalized odds, based only on a
second-moment condition instead of full conditional independence. We
show that under this second moment notion of non-discrimination it is
computationally tractable to learn a nearly optimal non-discriminatory
linear predictor with respect to a convex loss.

\section{Sub-optimality of post hoc correction} \label{sec:suboptimalityposthoc}
When the protected attribute $A$ and the target $Y$ are both binary, the post hoc correction algorithm proposed by \cite{hardt2016equality} can be applied to a binary or real-valued predictor $\hat{Y} \in \mathcal{H}$, deriving a randomized binary predictor that is non-discriminatory. The algorithm is convenient because it requires access only to the joint distribution over $(\hat{Y}, A, Y)$ and does not use the features $X$, thus it can be applied retroactively to an already trained predictor. Such predictors are formulated using the notion of a derived predictor:
\begin{definition} [Definition 4.1 in \cite{hardt2016equality}]\label{def:derivedpredictor}
A predictor $\tilde{Y}$ is \textbf{derived} from a random variable $R$ and protected attribute $A$ if it is a possibly randomized function of $(R,A)$ alone. In particular, $\tilde{Y}$ is independent of $X$ conditioned on $(R,A)$.
\end{definition}
For binary classification, the optimal post hoc correction $\tilde{Y}$ for a binary or real valued predictor $\hat{Y}\in \b{R}$ is simply the non-discriminatory, derived, binary predictor that minimizes the expectation of a loss $\ell$ over binary variables \citep{hardt2016equality}:
\begin{equation}
\begin{aligned} \label{eq:post-hocalgorithm}
\tilde{Y} = \argmin{f:\mathbb{R}\times\sset{0,1}\mapsto \sset{0,1}}&\ \mathbb{E}\,\ell\left( f(\hat{Y}, A), Y \right)\\
\textrm{s.t.}&\quad \gamma_{y0}(f) = \gamma_{y1}(f)\qquad \forall y = \sset{0,1}
\end{aligned}
\end{equation}
Two notable features of the corrected predictor $\tilde{Y}$ are that \textbf{a)} it is not constrained to any particular hypothesis class, and \textbf{b)} it may be a random function of $\hat{Y}$ and $A$; indeed for many distributions and hypothesis classes there may not even exist a non-constant, deterministic, non-discriminatory predictor. Nevertheless, $\tilde{Y}$  does indirectly depend on the hypothesis class $\c{H}$ from which $\hat{Y}$ was learned and  possibly a different loss over real valued variables used in training of $\hat{Y}$. 

We are interested in comparing the optimality of $\tilde{Y}$ from post hoc correction to the following $\Yfair$ which is the optimal non-discriminatory predictor in a hypothesis class $\c{H}$ under consideration:
\begin{equation}
\begin{aligned}
\Yfair = \argmin{h \in \mathcal{H}}\ \mathbb{E}\,\ell\left( h(X, A), Y \right)\qquad \textrm{s.t.}\quad \ \gamma_{y0}(h) = \gamma_{y1}(h)\qquad \forall y = \sset{0,1}
\end{aligned}
\end{equation}
Ideally, the expected loss of $\tilde{Y}$ would compare favorably against that of $Y^*$.
Indeed, \cite{hardt2016equality} show that when the target $Y$ is binary, if we can first find a predictor  $\hat{Y}$ that is exactly or nearly Bayes optimal for the squared loss over an unconstrained hypothesis class, then applying the post hoc correction \eqref{eq:post-hocalgorithm} using the 0-1 loss (i.e.~with $\ell = \ell^{01}$ in \eqref{eq:post-hocalgorithm}) to  $\hat{Y}$ will yield a predictor $\tilde{Y}$ that is non-discriminatory and has loss no worse than $Y^*$. This statement can be extended to the case of first finding the optimal \emph{unconstrained} predictor with resepect to any \emph{strictly convex} loss, and then using the post hoc correction \eqref{eq:post-hocalgorithm} with the {0-1 loss}. 

Nevertheless, from a practical perspective this approach is very unsatisfying. First, for general distributions, it is impossible to learn the Bayes optimal predictor from finite samples of data. Also, as we will show, the post hoc correction of even the optimal unconstrained predictor with respect to the 0-1 (non-convex) or even hinge (non-strict but convex) losses can have much worse performance than the best non-discriminatory predictor. Moreover, if the hypothesis class is restricted there can also be a gap between the post hoc correction of the optimal predictor in the hypothesis class and the best non-discriminatory predictor, even when optimizing a strictly convex loss function.

In the following example, we see that when the loss function is not strictly convex, the post hoc correction of even the unconstrained Bayes optimal predictor can have poor accuracy:
\begin{restatable}{example}{exampleone}\label{theorem:Step2LowerBound01}
When the hypothesis class is unconstrained, for any $\epsilon \in (0,1/4)$ there exists a distribution $\mathcal{D}_\epsilon$ such that \textbf{a)} the  optimal non-discriminatory predictor  $Y^*$  with respect to the 0-1 loss has loss at most $2\epsilon$ but \textbf{b)} for unrestricted Bayes optimal predictor  $\hat{Y}$ trained on  0-1 loss, the post hoc correction of $\hat{Y}$ with 0-1 loss returns a predictor $\tilde{Y}$ with loss at least $0.5$. 

A similar statement can also be made about predictors trained on hinge loss. For an unconstrained hypothesis class, for any $\epsilon \in (0,1/4)$ and the same distribution $\mathcal{D}_\epsilon$, \textbf{a)} the  optimal non-discriminatory predictor $Y^*$ with respect to the hinge loss has loss at most $4\epsilon$ but \textbf{b)} the post hoc correction of the Bayes optimal unrestricted predictor trained on hinge loss has loss $1$.
\end{restatable}
\noindent
We construct $\mathcal{D}_\epsilon$ as follows:\\
\noindent
\begin{minipage}{.2\textwidth}
\vspace{2mm}
\begin{tikzpicture}
\node[shape=circle,draw=black] (X) at (2.5,0) {$X$};
\node[shape=circle,draw=black] (Y) at (1,0.6) {$Y$};
\node[shape=circle,draw=black] (A) at (2.5,1.2) {$A$};

\path [->] (Y) edge node[left] {} (X);
\path [->] (Y) edge node[left] {} (A);
\end{tikzpicture} 
\vspace{2mm}
\end{minipage}
\begin{minipage}{.8\textwidth}
\vspace{-5mm}
\begin{equation} \label{eq:Depsilon}
\begin{aligned}
&X,A,Y \in \sset{0,1} \qquad\qquad &&\mathbb{P}_{\mathcal{D}_\epsilon}(Y = 1) = \frac{1}{2}\\
&\mathbb{P}_{\mathcal{D}_\epsilon}(A = y\ |\ Y = y) = 1 - \epsilon  \qquad &&\mathbb{P}_{\mathcal{D}_\epsilon}(X = y\ |\ Y = y) = 1 - 2\epsilon
\end{aligned}
\end{equation}
\end{minipage}

\noindent
Both $X$ and $A$ are highly predictive of $Y$, but $A$ is slightly more so. Therefore, minimizing either the 0-1 or the hinge loss, without regard for non-discrimination, returns $\hat{Y}=A$ and ignores $X$ entirely. Consequently, $\gamma_{y1}(\hat{Y}) = 1$ and $\gamma_{y0}(\hat{Y})=0\neq \gamma_{y1}(\hat{Y})$ so the Bayes optimal predictor is discriminatory, and the post hoc correction, which is required to be non-discriminatory and derived from $\hat{Y}=A$, is forced to return a constant predictor even though returning $Y^*=X$ would be accurate and non-discriminatory. A more detailed proof is included in Appendix \ref{appendix:sec2-ex1}

In the second example, we show that when the hypothesis class is restricted, the correction of the optimal regressor in the class can yield a suboptimal classifier, even with squared loss. 
\begin{restatable}{example}{exampletwo}\label{theorem:Step2LowerBoundConvex}
Let $\mathcal{H}$ be the class of linear predictors with $L^1$ norm at most $\frac{1}{2} - 2\epsilon$, for some $\epsilon \in (2/25,1/4)$. There exists a distribution $\mathcal{D}_\epsilon$ such that \textbf{a)} the optimal non-discriminatory predictor  in  $\mathcal{H}$ with respect to the squared loss has square loss at most $\frac{1}{16} +\frac{3\epsilon}{2} + 3\epsilon^2$, but \textbf{b)} the post hoc correction of the Bayes optimal square loss regressor in $\mathcal{H}$ returns a constant predictor which has (trivial) square loss of $1/4$.

Similarly, for the class $\mathcal{H}$ of sparse linear predictors, for any $\epsilon \in (0,1/4)$, there exists a distribution $\mathcal{D}_\epsilon$ such that \textbf{a)} the optimal non-discriminatory predictor in $\c{H}$ with respect to the squared loss has square loss at most $2\epsilon - 4\epsilon^2$, but \textbf{b)} the post hoc correction of the Bayes optimal squared loss regressor  in $\c{H}$ again returns a constant predictor which has (trivial) square loss of $1/4$.
\end{restatable} 

The distribution $\mathcal{D}_\epsilon$ is the same as was defined in \eqref{eq:Depsilon}. Again, $A$ is slightly more predictive of $Y$ than $X$, and since the sparsity or the sparsity surrogate $L^1$ norm of the predictor is constrained by the hypothesis class, the Bayes optimal predictor chooses to use just the feature $A$ and ignore $X$. Consequently, the optimal predictor is extremely discriminatory, and the post hoc correction algorithm will return a highly sub-optimal constant predictor which performs no better than chance. Details of the proof are deferred to Appendix \ref{appendix:sec2-ex2}

From these examples, it is clear that simply finding the optimal predictor with respect to a particular loss function and hypothesis class and correcting it post hoc can perfom very poorly. We conclude that in order to learn a predictor that is simultaneously accurate \emph{and} non-discriminatory in the general case, it is essential to account for non-discrimination during the learning process. 

\section{Detecting Discrimination in Binary Predictors} \label{sec:detection}
 In the following sections, we look at tools for integrating non-discrimination  into the supervised learning framework. In formulating algorithms for learning non-discriminatory predictors, it is important to consider the non-asymptotic behavior under finite samples.
Towards this, one of the first issues to be addressed is that, using finite samples it is not feasible to ensure, or even verify if  a predictor $\hat{Y}$ satisfies the non-discrimination criterion in Definition~\ref{def:non-discrimination}.  
This necessitates defining a notion of approximate non-discrimination which can be computed using finite samples and which asymptotically generalizes to  the equalized odds criteria in the population. 

Let us consider the task of binary classification, where both $A,Y \in \sset{0,1}$ and the predictors $\hat{Y}$ output values in $\sset{0,1}$. Recall the definition of the population group-conditional true and false positive rates $\gamma_{ya}(\hat{Y}) = \mathbb{P}(\hat{Y} = 1\ |\ Y = y, A = a)$ and the fact that non-discrimination is equivalent to satisfying $\gamma_{00} = \gamma_{01}$ and $\gamma_{10} = \gamma_{11}$. 

For a set of of $n$ i.i.d.~samples,  $S = \sset{(x_i,a_i,y_i)}_{i=1}^n\sim\b{P}^n(X,A,Y)$, the sample analogue of $\gamma_{ya}$ is defined as follows, 
\begin{equation}
\gamma^S_{ya}(\hat{Y})=\frac{1}{n^S_{ya}}\sum_{i=1}^n\hat{Y}(x_i, a_i)\mathbf{1}(y_i = y, a_i = a)\text{, where }n^S_{ya} = \sum_{i=1}^n \mathbf{1}(y_i = y, a_i = a).
\label{eq:gamma_sample}
\end{equation}
To ensure non-discrimination, we could possibly require $\gamma^S_{y0}=\gamma^S_{y1}$ on a large enough sample $S$, however this not ideal for two reasons.  First, even when $\gamma^S_{y0} = \gamma^S_{y1}$ on $S$, this almost certainly does not ensure that $\gamma_{y0} = \gamma_{y1}$ on the population. For this same reason, it is impossible to be certain that a given predictor is  non-discriminatory on the population. Moreover, if $n^S_{y0}\neq n^S_{y1}$, it is typically not feasible  to match $\gamma^S_{y0} = \gamma^S_{y1}$ for non-trivial predictors, e.g.~if~$n^S_{y0} = 2$ and $n^S_{y1} = 3$, then $\gamma^S_{y0} \in \sset{0, \frac{1}{2}, 1}$ but $\gamma^S_{y1} \in \sset{0, \frac{1}{3}, \frac{2}{3}, 1}$, thus the only predictors with $\gamma^S_{y1}=\gamma^S_{y0}$ would be the ones which are constant conditioned on $Y$, i.e.  the constant predictors $\hat{Y}=0$ and $\hat{Y}=1$,  or the perfect predictor $\hat{Y}=Y$.

For these reasons, we define the following notion of approximate non-discrimination, which \emph{is} possible to ensure on a sample and, when it holds on a sample, generalizes to the population.
\begin{definition}[$\alpha$-discrimination]\label{def:approxnondiscrimination}\!\!
A possibly randomized binary predictor $\hat{Y}$ is {$\alpha$-discriminatory} with respect to a binary protected attribute $A$ on the population or on a sample $S$ if, respectively, 
\begin{equation}
\Gamma(\hat{Y}) := \max_{y \in \sset{0,1}} \abs{\gamma_{y0}(\hat{Y}) - \gamma_{y1}(\hat{Y})} \leq \alpha \quad \textrm{or}\quad\Gamma^S(\hat{Y}) := \max_{y \in \sset{0,1}} \abs{\gamma^S_{y0}(\hat{Y}) - \gamma^S_{y1}(\hat{Y})} \leq \alpha.\end{equation}
\end{definition}
The decision to define approximate non-discrimination in terms of \emph{conditional} rather than joint probabilities is important, particularly in the case that the $a,y$ pairs occur with widely varying frequencies. For example, if approximate non-discrimination were defined in terms of the joint probabilities $P(\hat{Y} = \hat{y}, A = a, Y = y)$ and if $P(A = 0, Y = 1) = \alpha / 10$, then a predictor could be ``$\alpha$-discriminatory'' all while being arbitrarily unfair towards the $A = 0, Y = 1$ population. This issue does not arise when using Definition \ref{def:approxnondiscrimination} and it incentivizes  collection of sufficient data for minority groups to ensure non-discrimination.

For Definition~\ref{def:approxnondiscrimination}, we propose a simple statistical test to test the hypothesis that a given predictor $\hat{Y}$ is at most $\alpha$-discriminatory on the population for some $\alpha > 0$. 
Let $S = \sset{(x_i,a_i,y_i)}_{i=1}^n\sim\b{P}^n(X,Y,A)$ denote a set of $n$ i.i.d. samples, and for $y,a\in\{0,1\}$, let $\P_{ya}=\b{P}(Y=y,A=a)$. 
We propose the following test for detecting $\alpha$-discrimination:
\begin{equation}
T\left(\hat{Y}, S,\alpha\right) = \mathbf{1}\left( \Gamma^S(\hat{Y}) > \alpha \right)
\end{equation}
\begin{restatable}{lemma}{detectiontest}\label{lem:detectiontest}
Given $n$ i.i.d.~samples $S$, $\forall \alpha\in(0,1),\delta\in(0,1/2)$, if $n > \frac{16\log{32/\delta}}{\alpha^2\min_{ya}\P_{ya}}$, then  with probability greater than $1-\delta$, $T$ satisfies,
\[T\left(\hat{Y}, S,\frac{\alpha}{2}\right) = \begin{cases} 0 & \textrm{if } \hat{Y} \textrm{ is 0-discriminatory on population} \\ 1 & \textrm{if } \hat{Y} \textrm{ is at least } \alpha \textrm{-discriminatory on population.} \end{cases}
\]
\end{restatable}
\noindent The proof is based on  the following concentration results  for $\Gamma^S$ and is provided in Appendix \ref{appendix:detection}.
\begin{lemma} \label{lemma:bin_step1} For $\delta\in(0,1/2)$ and a binary predictor $h$, if $n>\frac{8\log{8/\delta}}{\min_{ya}\P_{ya}}$, then
\begin{equation*}
\begin{split}
&\b{P}\left(\left|\Gamma(h)-\Gamma^S(h)\right|> 2\max_{ya}\sqrt{\frac{\log{16/\delta}}{n\P_{ya}}}\right)\le \delta.
\end{split}
\end{equation*}
\end{lemma}

\section{Learning Optimal Non-discriminatory Binary Predictors}\label{sec:binary}
In Example \ref{theorem:Step2LowerBound01}, we saw that even though an almost perfect non-discriminatory predictor exists within the hypothesis class, if we ignore non-discrimination in training with  $0$-$1$ loss, then optimal post hoc correction of even the Bayes optimal predictor using \eqref{eq:post-hocalgorithm} yields a poor predictor with no better than chance accuracy.
Thus, to find a predictor that is both  nearly non-discriminatory \emph{and} has nearly optimal loss for general hypothesis classes, it is necessary to incorporate non-discrimination into the learning process. For a hypothesis class $\c{H}$, we would ideally like to find the optimal non-discriminatory predictor:
\begin{equation}
\begin{aligned}
\Yfair = \argmin{h \in \mathcal{H}}\ \mathbb{E}\ell\left( h(X, A), Y \right) \qquad \textrm{s.t.}\quad \ \gamma_{y0}(h) = \gamma_{y1}(h)\qquad \forall y = \sset{0,1}.
\end{aligned}
\end{equation}
However, as discussed in Section~\ref{sec:detection}, it is impossible to learn $0$-discriminatory predictors from finite samples. 
In this section, we address following question: {\emph{given $n$ i.i.d~samples $S$, what level of approximate non-discrimination and accuracy is it possible to ensure in a learned predictor?}}

We propose a two step framework for learning a non-discriminatory binary predictor that  minimizes the expected 0-1 loss $\c{L}(\hat{Y})=\b{E}\,\ell^{01}(\hat{Y},Y) =\b{E}\,\mathbf{1}(\hat{Y}\ne Y)$ over a binary hypothesis class $\c{H}=\{h:\c{X}\to\{0,1\}\}$. 
Broadly, the two-step framework is as follows:
\begin{compactenum}
\item \textbf{Non-discrimination in training}: Estimate an almost non-discriminatory empirical risk minimizer $\hat{Y}$ by incorporating approximate non-discrimination constraints on the samples.
\item \textbf{Post-training correction}: With additional samples from $(\hat{Y},Y,A)$, derive a randomized predictor $\tilde{Y}$ to further reduce discrimination.
\end{compactenum}

\subsection{Two step framework for binary predictors} 
We partition the training data consisting of $n$ independent samples ${S}=\{(x_i,a_i,y_i)\sim\bP(X,A,Y)\}_{i=1}^n$ into two subsets $S_1$ and $S_2$ to be used in Step~$1$ and Step~$2$, respectively\footnote{We occasionally overload the $S$ to also denote the indices $[n]$, e.g. $i\in S$ to denote the $i^\text{th}$ sample $(x_i,a_i,y_i)\in S$.}. 
For a predictor $h\in\c{H}$ and  $y,a\in\{0,1\}$,  recall notation for the population and sample group-conditional true and false positive rates $\gamma_{ya}(h)$ and $\gamma_{ya}^S(h)$ from \eqref{eq:gammapop} and \eqref{eq:gamma_sample}, respectively. Additionally, for $S_1$ and $S_2$, let  $n^{S_k}_{ya}=\sum_{i\in S_k}\mathbf{1}(y_i=y,a_i=a)$. In general the subsets \textit{need not} be of equal size, but for simplicity, let $\abs{S_1} = \abs{S_2}=n/2$.

\subsection*{Step 1: Non-discrimination in training}
For the first step, we estimate an empirical risk minimizing predictor  $\hat{Y} \in \mathcal{H}$, subject to the constraint that $\hat{Y}$ be $\alpha_{n}$-discriminatory on $S_1$, where $\alpha_{n}$ is a tunable hyperparameter:
\begin{equation} \label{eq:step1}
\begin{aligned}
\hat{Y}=\argmin{h \in \mathcal{H}}&\ \frac{2}{n}\sum_{i \in S_1} \ell^{01}(h(x_i),y_i) \\
\textrm{s.t.}&\quad\Gamma^{S_1}(h)= \max_{y\in\{0,1\}} \left|\gamma^{S_1}_{y0}(h) - \gamma^{S_1}_{y1}(h)\right| < \alpha_{n}.
\end{aligned}
\end{equation}


\subsection*{Step 2: Post-training correction}
As Step $2$, we propose a post hoc correction to $\hat{Y}$ estimated in Step $1$  for improved non-discrimination. Using samples in $S_2$, which are independent of $S_1$, we estimate the best randomized predictor $\tilde{Y}$ derived from $(\hat{Y},A)$ (Definition~\ref{def:derivedpredictor}). Let $\c{P}(\hat{Y})$ denote the set of randomized binary predictors that can be derived solely from $\bP(\hat{Y},A,Y)$ and let $\tilde{\alpha}_n$ be a tunable hyperparameter, then $\tilde{Y}$ is given by,
\begin{equation}
\begin{aligned}
\tilde{Y}=\argmin{\tilde{Y} \in \mathcal{P}(\hat{Y})}&\ \frac{2}{n}\sum_{i\in S_2} \b{E}_{\tilde{Y}}\ell^{01}(\tilde{Y}(\hat{y}_i,a_i),y_i) \\
\textrm{s.t.}&\quad \Gamma^{S_2}(\tilde{Y})= \max_{y\in\{0,1\}} \left|\gamma^{S_2}_{y0}(\tilde{Y}) - \gamma^{S_2}_{y1}(\tilde{Y})\right| < \tilde{\alpha}_n,
\end{aligned}
\label{eq:step2}
\end{equation}
where for a randomized predictor ${\tilde{Y}}$, the group-conditional probabilities on a sample is defined to be $\gamma_{ya}^{S_2}(\tilde{Y})=\frac{1}{n^{S_2}_{ya}}\sum_{i\in S_2}\b{E}_{\tilde{Y}}\mathbf{1}(\tilde{Y}_i=1,Y_i=y,A_i=a).$ The above optimization problem is a finite sample adaptation of the post hoc correction  in \eqref{eq:post-hocalgorithm} proposed by \cite{hardt2016equality}.

As with the post hoc correction on the population \eqref{eq:post-hocalgorithm}, estimating a predictor $\tilde{Y}\in\c{P}(\hat{Y})$ derived from $(\hat{Y},A)$ is a simple optimization over the following four parameters that completely  specify $\tilde{Y}$, 
\begin{equation}
\tilde{p}_{\hat{y}a}:=\tilde{p}_{\hat{y}a}(\tilde{Y})=\bP(\tilde{Y}=1\ |\ \hat{Y}=\hat{y},A=a)\text{ for }\hat{y},a\in\{0,1\}.
\end{equation}

In Section~\ref{sec:stat}, we discuss how the post hoc correction step offers statistical advantages over the one-shot approach of using all of the training data for Step $1$. Besides these statistical advantages, the post hoc correction step is also motivated from other practical considerations:
\begin{inparaenum}[(a)]
\item the derived predictors only need access to samples from $\bP(\hat{Y},Y,A)$  and can be deployed without explicit access to predictive features $X$, and 
\item the proposed correction step \eqref{eq:step2} can be easily optimized using ternary search and can be repeated multiple times as more and more samples from $\bP(\hat{Y},Y,A)$ are seen by the system, without having to retrain the classifier from scratch.
\end{inparaenum}


\subsection{Statistical guarantees}\label{sec:stat}
In this section, we discuss the statistical properties of the estimators $\hat{Y}$ and $\tilde{Y}$ from Step $1$ and Step $2$, respectively. We define the following notation for succinctly describing the quality of a predictor:
\begin{definition} $\c{Q}(L,\alpha)=\{h:\c{X}\to\{0,1\}: \Gamma(h)\le \alpha, \c{L}(h)\le L\}$ denotes the set of $\alpha$-discriminatory binary predictors with loss $L$. 
\end{definition}
The following theorem shows the statistical learnability of hypothesis classes $\c{H}$ with respect to the best non-discriminatory predictor in $\c{H}$ using the two-step framework. 
In the following results,  for $y,a\in\{0,1\}$,  recall the notation  $\P_{ya}=\b{P}(Y=y,A=a)$ for the group-outcome probabilities. Additionally, for $a\in\{0,1\}$, and let $VC(\c{H})$ denote the Vapnik-Chervonenkis dimension of a hypothesis class $\c{H}$. 

\begin{restatable}{theorem}{thmub}\label{theorem:UpperBounds} Let $n\!=\!\Omega\big(\!\!\max_{ya}\frac{\!\log{1/\delta}}{\P_{ya}}\!\big)$ and the hyperparameters satisfy $\alpha_n,\tilde{\alpha}_n\!=\!\Theta\Big(\!\!\max_{ya}\!\sqrt{\!\frac{\log{1/\delta}}{n\P_{ya}}}\Big)$. 

For a binary hypothesis class $\c{H}$, any distribution $\b{P}(X,Y,A)$, and any $\delta\in(0,1/2)$, if $\Yfair\in\c{H}$ is a non-discriminatory predictor, then  with probability greater than $1-\delta$, the output of the two step procedure $\tilde{Y}$ satisfies the following for absolute constants $C_1$ and $C_2$, 
$$\c{L}(\tilde{Y}) \leq \c{L}(Y^*)+ C_1\Db \text{, and }\Gamma(\tilde{Y}) \le C_2\max_{ya}{\epsb}.$$ 

Thus, with $n =\Omega\Big(\frac{VC(\c{H})}{\epsilon^2}+\frac{1}{\alpha^2}\Big) \frac{1}{\min_{ya}\P_{ya}}$ samples, and an appropriate choice of $\alpha_n,\tilde{\alpha}_n$,  the  two step framework returns $\tilde{Y}\in \c{Q}(\c{L}(Y^*)+\epsilon,\alpha)$ with high probability.
\end{restatable}
The proof is based on the following two Lemmas.  The first is a statistical guarantee on the  loss and non-discrimination after training in Step $1$: 
\begin{lemma}
\label{lemma:BinaryStep1Guarantee} Under the conditions in Theorem~\ref{theorem:UpperBounds}, if $\alpha_n>2\max_{ya}\sqrt{\frac{\log{64/\delta}}{n\P_{ya}}}$, then w.p. greater than $1-\delta$, $\hat{Y}$ from Step $1$ satisfies 
$$\c{L}(\hat{Y}) \leq \c{L}(Y^*) + C_1\Da \text{, and }\Gamma(\hat{Y}) \le \alpha_n+ C_2\epsa.$$
\end{lemma}

The following second  lemma ensures that if  $\hat{Y}$ from Step $1$ is approximately non-discriminatory, the correction in the second step does not incur significant additional loss:
\begin{lemma}\label{steptwotopop}
If $h$ is an $\alpha$-discriminatory binary predictor $h\in\c{Q}(\c{L}(h),\alpha)$, then the optimal 0-discriminatory derived predictor $\tilde{Y}^*(h)$  from \eqref{eq:post-hocalgorithm} using 0-1 loss satisfies $\tilde{Y}^*(h)\in\c{Q}(\c{L}(h)+\alpha,0)$.
\end{lemma}
This lemma, along with the examples in Section~\ref{sec:suboptimalityposthoc} further motivates an integrated learning step, such as Step $1$, where non-discrimination is explicitly encouraged in training. 

Notice that, unlike to the guarantees for $\tilde{Y}$ in Theorem~\ref{theorem:UpperBounds}, the upper bound on non-discrimination for $\hat{Y}$ in Lemma~\ref{lemma:BinaryStep1Guarantee}  scales with the complexity of the hypothesis class $VC(\c{H})$. In the Step~$2$, we search over a much  restricted space  of derived predictors $\c{P}(\hat{Y})$ which is essentially the convex hull of $|\c{A}|$ conditional predictors of $\hat{Y}$. This  allows us to obtain a guarantee on non-discrimination for $\tilde{Y}$ that does not scale with $VC(\c{H})$.  See Appendix~\ref {app:ub} for the remainder of the proof. 

Finally, although Lemma~\ref{lemma:BinaryStep1Guarantee} is only an upper bound on non-discrimination from Step $1$, the following theorem shows that the level of non-discrimination using \textit{just} the first step of the procedure  can indeed grow with the complexity of the hypothesis class. 
\begin{theorem} \label{theorem:Step1LowerBound} 
There exists a finite, binary hypothesis class $\mathcal{H}$ and a data distribution $\mathcal{D}$ such that with probability at least $1/2$, the classifier $\hat{Y}$ learned from Step 1 using $n$ samples from $\mathcal{D}$ is at least $\max_{y,a}\frac{3\log \frac{\abs{\mathcal{H}} - 1}{5}}{4n\P_{ya}}$-discriminatory on the population.
\end{theorem} 

Theorem~\ref{theorem:Step1LowerBound} suggests that the non-discrimination guarantee provided by Lemma \ref{lemma:BinaryStep1Guarantee} has the correct dependence on the problem parameters. 
Thus, for the intermediate predictor $\hat{Y}$ from Step $1$, without the post hoc correction, the  tolerance for non-discrimination that can be guaranteed grows with the complexity of the hypothesis class (Lemma~\ref{lemma:BinaryStep1Guarantee}). 
On the other hand, when using the two step procedure, the sample complexity of ensuring that the final predictor $\tilde{Y}$ is at most $\alpha$-discriminatory is $\Omega(\max_{ya} \,1/\alpha^{2}\P_{ya})$ (Theorem~\ref{theorem:UpperBounds}) which \textit{does not} depend on the complexity of the hypothesis class, and also matches the sample complexity of merely detecting $\alpha$-discrimination from Lemma~\ref{lem:detectiontest}.

We also note that the sample complexity dependence on $\P_{ya}$ in Theorem~\ref{theorem:UpperBounds} is unavoidable for our definition of approximate non-discrimination. If there is a rare group or group-outcome combination, we still need enough samples from that group to ensure that the loss and sample conditional distributions $\gamma_{ya}^S$ generalize to the population for every $A=a,Y=y$. This is the same reason why the dependence on $\P_{ya}$ arises in Lemma~\ref{lem:detectiontest}. This bottleneck provides further incentive to actively seek samples and target labels for  minority populations, which might otherwise be disregarded if non-discrimination were not a consideration.

\section{Computational Intractability of Learning Non-discriminatory Predictors} \label{sec:hardness}
The proposed procedure for learning non-discriminatory predictors from a finite sample is statistically optimal, but it is clearly computationally intractable for almost any interesting hypothesis class since the first step \eqref{eq:step1} involves minimizing the 0-1 loss. As is typically done with intractable learning problems, we therefore look to alternative loss functions and hypothesis classes in order to find a computationally feasible procedure.

A natural choice is the hypothesis class of real valued linear predictors with a convex loss function. In this case, we would like to have an efficient algorithm for finding a non-discriminatory predictor that has convex loss that is approximately as good as the loss of the best non-discriminatory linear predictor. However, even in the case of binary $A$ and $Y$ and even with a convex loss, for real-valued predictors $h(x)\in\b{R}$, the non-discrimination constraint in Defination~\ref{def:non-discrimination} is {extremely} strong, requiring that the group-conditional true and false positive rates match at \textit{every} threshold. In fact, the mere existence of a non-trivial (i.e.~non-constant) linear predictor that is non-discriminatory requires a relatively special distribution. This is the case even when considering a real-valued analogue of $\alpha$-approximate non-discrimination.

For binary targets $Y\in\{0,1\}$, one could relax the problem one step further with a less restrictive non-discrimination requirement. Consider the class of linear predictors with a convex loss where only the sign of the predictor need be non-discriminatory. Unfortunately, using a result by \cite{daniely2015complexity} even this is computationally intractable:
\begin{theorem}\label{thm:hardness}
Let $L^*$ be the hinge loss of the optimal linear predictor whose sign is non-discriminatory. Subject to the assumption that refuting random K-XOR formulas is computationally hard,\footnote{See \cite{daniely2015complexity} for a description of the problem.} the learning problem of finding a possibly randomized function $f$ such that $\mathcal{L}^{\textrm{hinge}}(f) \leq L^* + \epsilon$ and $\textrm{sign(f)}$ is $\alpha$-discriminatory requires exponential time in the worst case for $\epsilon < \frac{1}{8}$ and $\alpha < \frac{1}{8}$.
\end{theorem}
The proof goes through a reduction from the hardness of improper, agnostic PAC learning of \textsc{Halfspaces}. 
Given a distribution $\mathcal{D}$ over $(X,Y)$ and the knowledge that there is a linear predictor which achieves 0-1 loss $\ell^*$ on $\mathcal{D}$, we construct a new distribution $\tilde{D}$ over $(\tilde{X}, \tilde{A}, \tilde{Y})$ such that an approximately non-discriminatory predictor with small hinge loss can be used to make accurate predictions on $\mathcal{D}$, even if it is not a linear function. The distribution $\tilde{\mathcal{D}}$ is identical to the original distribution $\mathcal{D}$ when conditioned on $\tilde{A} = 1$, and is supported on only two points conditioned on $\tilde{A} = 0$. The probabilities of the two points are constructed so that satisfying non-discrimination requires making accurate predictions on the $\tilde{A} = 1$ population, and thus on $\mathcal{D}$. In particular, for parameters $\epsilon,\alpha < \frac{1}{8}$, the predictor will have 0-1 loss at most $\frac{15}{16}\ell^* + \frac{47}{128}$ on $\mathcal{D}$, which is bounded away from $\frac{1}{2}$ when $\ell^* < \frac{1}{10}$. Since \citet{daniely2015complexity} proves that finding a predictor with accuracy bounded away from $\frac{1}{2}$ is hard in general, we conclude that the learning problem is computationally hard. See Appendix \ref{appendix:hardness} for a complete proof.

To summarize, learning a non-discriminatory binary predictor with the 0-1 loss is hard; learning a real-valued linear predictor with respect to a convex loss function is problematic due to the potential inexistence of a non-trivial non-discriminatory linear predictor; and even requiring only that the sign of the linear predictor be non-discriminatory is computationally hard. A more significant relaxation of non-discrimination is therefore required to arrive at a computationally tractable learning problem.

\newcommand{\E}{\mathbb{E}}
\newcommand{\tp}{^T}
\newcommand{\uv}{\mathrm{v}}
\newcommand{\XA}{\smatrix{X\\A}}
\newcommand{\Cov}[1]{\Sigma_{#1}}
\newcommand{\cov}[1]{\sigma_{#1}}
\newcommand{\var}[1]{\sigma^2_{#1}}
\newcommand{\weight}{w}
\newcommand{\editremark}[1]{\textbf{[#1]}}
\newcommand{\smatrix}[1]{\left[\begin{smallmatrix}#1\end{smallmatrix}\right]}

\section{Relaxing Non-discrimination for Tractable Estimation} \label{sec:2ndorder}

Motivated by the hardness result in Section \ref{sec:hardness}, we now proceed to relax the criterion of equalized odds. More precisely, we seek to identify a more tractable notion of non-discrimination based on \emph{second order moments}. The most similar prior work to this is given by \cite{ZafarVGG16}, who suggest a notion of non-discrimination based on \emph{first order} moments. However, this may be viewed as a constraint on second order moments provided that $A$ is binary. Namely, their notion amounts to relaxing the equalized odds constraint that $\mathbb{P}(\hat{Y} = \hat{y} \ |\ Y = y, A = 0) = \mathbb{P}(\hat{Y} = \hat{y} \ |\ Y = y, A = 1)$ to the constraint that $\mathbb{E}[\hat{Y} \ |\ Y = y, A = 0] = \mathbb{E}[\hat{Y} \ |\ Y = y, A = 1]$. They propose optimizing a convex loss subject to an approximation of this constraint. Their work is primarily applied and gives no learning or non-discrimination guarantees for their learning rule, both of which we address in this section for a different relaxation.

In particular we propose the notion of \emph{equalized correlations}. Equalized correlations is generally a weaker condition than equalized odds, but when considering the squared loss and when $X,A,Y$ are jointly Gaussian, it is in fact equivalent.

\begin{definition}[Equalized correlations] \label{def:equalized-correlations}
We say that a real-valued predictor $R$ satisfies \emph{equalized correlations} or is \emph{second-moment non-discriminatory} with respect to the protected attribute $A$ and target $Y$, if $R$, $A$, and $Y$ satisfy the following:
\begin{equation}\label{eq:equalized-correlations}
    \cov{RA}\var{Y}=\cov{RY}\cov{YA},
\end{equation}
where $\cov{\alpha\beta} = \mathbb{E}\left[ (\alpha - \mathbb{E}[\alpha])(\beta - \mathbb{E}[\beta]) \right]$ is the covariance between two random variables $\alpha$ and $\beta$.
\end{definition}

\begin{theorem}[Gaussian non-discrimination] \label{thm:gaussian}
If $X$, $A$, and $Y$ are jointly Gaussian, then the squared loss optimal (equalized odds) non-discriminatory predictor is linear. Furthermore, for all linear predictors in this setting, non-discrimination is equivalent to satisfying \eqref{eq:equalized-correlations}.
\end{theorem}

Theorem \ref{thm:gaussian} means in particular that in the Gaussian setting under the squared loss, the optimal predictor that is non-discriminatory in the sense of Definition \ref{def:non-discrimination} is the same as the optimal predictor that is non-discriminatory in the more relaxed sense of Definition \ref{def:equalized-correlations}. We stress that this is generally not the case for non-Gaussian scenarios and with losses other than the squared loss that may result in non-linear optimal predictors.

One may also consider intermediate notions of non-discrimination between equalized odds and equalized correlations. One such option is to require that the conditional covariance $\cov{RA|Y}$ vanishes. We do not elaborate further, except to note that it is immediate that equalized odds implies this to be the case and this in turn generally implies equalized correlations.

Since linear prediction can be thought of as the hallmark of Gaussian processes, one could therefore justify the relaxation of Definition \ref{def:equalized-correlations} as being the appropriate notion of non-discrimination when we restrict our predictors to be linear. Linear predictors, especially under kernel transformations, are used in a wide array of applications. They thus form a practically relevant family of predictors where one would like to achieve non-discrimination. Therefore, in the remainder of this section, we develop a theory for non-discriminating linear predictors.

A particularly attractive feature of equalized correlations is that, for linear predictors, Equation \eqref{eq:equalized-correlations} amounts to a linear constraint. With any convex loss $\ell(\cdot,\cdot)$, finding the optimal second-moment non-discriminatory predictor can be written as (where $\Cov{\cdot,\cdot}$ are covariance matrices):
\begin{align*}
\min_{\weight}\ \E\left[\ell\left(Y,\weight\tp\XA\right)\right] \qquad\textrm{s.t.}\quad\weight\tp\Big(\Cov{\XA,A}\var{Y}-\Cov{\XA,Y}\cov{YA}\Big)=0
\end{align*}
This is a convex optimization problem with a single linear constraint, and is thus generally tractable.

In what follows, we take $X$ to be a real-valued vector, $A$ to be a scalar,  and the target to be binary $Y\in\{\pm 1\}$, unless otherwise noted. We also commit to the squared loss $(R-Y)^2$ throughout. Without loss of generality, we assume that $X$, $A$, and $Y$ all have zero mean. Recall the definition of the \emph{optimal linear predictor}:
\begin{equation} \label{eq:opt-lin}
    \widehat R=\altargmin_{r(x,a)=\weight\tp\smatrix{x\\a}} \E[(Y-r(X,A))^2]=\widehat r(X,A),
\end{equation}
where one can determine that $\widehat r(x,a)={\widehat \weight}\tp\smatrix{x\\a}$ with
\[
\widehat \weight=\E\left[\XA\smatrix{X\tp&A\tp}\right]^{-1}\E\left[\XA Y\right] = \Cov{\XA,\XA}^{-1}~\Cov{\XA,Y}.
\]

The optimal linear predictor may, in general, be discriminatory, so we define the \emph{optimal second-moment non-discriminatory linear predictor} as follows:
\begin{equation} \label{eq:opt-eq-cor-lin}
    R^\star=\altargmin_{r(x,a)=\weight\tp\smatrix{x\\a}} \E[(Y-r(X,A))^2]= r^\star(X,A) \qquad\textrm{s.t.}\quad \cov{r(X,A),A}\var{Y}=\cov{r(X,A),Y}\cov{YA}
\end{equation}
\noindent
thus the predictor is constrained to be linear and to satisfy \eqref{eq:equalized-correlations}, which is a single linear constraint. We can give a closed-form expression for $R^\star$ (see Section \ref{proof:opt-eq-cor-lin} in Appendix \ref{appendix:2ndorder} for details). We have that $r^\star(x,a)= {\weight^\star}\tp \smatrix{x\\a}$, where
\[
    \weight^\star =  \Cov{\XA,\XA}^{-1} ~ \left(\Cov{\XA,Y} -\alpha\uv \right)
\]
where $\uv$ is a vector encoding the non-discrimination constraint and $\alpha$ is a scalar defined as
\[
    \uv = \Cov{\XA,A} - \Cov{\XA,Y}\cov{YA}/\var{Y} \qquad\textrm{and}\qquad\alpha=\frac{\uv\tp\Cov{\XA,\XA}^{-1}\Cov{\XA,Y}}{\uv\tp\Cov{\XA,\XA}^{-1}\uv}
\]

Written as such, $R^\star$ is a function of $X$ and $A$. Nevertheless, it turns out that this optimal non-discriminatory linear predictor can be derived, in the sense of Definition \ref{def:derivedpredictor}, from the optimal (possibly discriminatory) linear predictor $\widehat R$ of Equation \ref{eq:opt-lin} and without access to $X$ individually.

\begin{theorem}[Derived] \label{thm:no-gap}
The second-moment non-discriminatory linear predictor minimizing the squared loss can be derived from the optimal least squares linear predictor $\widehat R$ and the joint (second moment) statistics of $(\widehat R,A,Y)$. Specifically, 
\[
    R^\star = \widehat R - \alpha\left(A-\widehat R\cov{YA}/\var{Y}\right),
\]
with
\[
        \alpha=\frac{\cov{YA} - \cov{\widehat R,Y}\cov{YA}/\var{Y}}{\var{A} -2(\cov{YA})^2/\var{Y}+\cov{\widehat R,Y}(\cov{YA})^2/(\var{Y})^2}.
\]
\end{theorem}

Theorem \ref{thm:no-gap} shows that, as far as the equalized correlation criterion is concerned, there is no penalty for first finding an optimal linear predictor and then correcting it. Consequently, this criterion easily enforceable on existing predictors. Intuitively, one must simply ``subtract'' any potential correlation one could derive about protected attributes from the prediction score and the outcome, and this can be entirely determined from the statistics of the optimal linear predictor $\widehat R$, the protected attribute $A$, and the outcome $Y$, without the need to know the extended set of attributes $X$. We emphasize that this result does not rely on any Gaussian assumptions, but simply on the fact that we have limited ourselves to linear predictors and the relaxed notion of non-discrimination. Finally, it is worth mentioning that a two-step procedure, as in Section \ref{sec:binary}, could be developed also for learning second-moment (approximately) non-discriminatory linear predictors from samples.

\section{Conclusion}

In this work we took the first steps toward a statistical and computational theory of learning non-discriminatory (equalized odds) predictors. We saw that post hoc correction might not be optimal and devised a statistically optimal two-step procedure, after observing that a straightforward ERM-type approach is not sufficient.  Computationally, working with binary non-discrimination is essentially has hard as agnostically learning binary predictors, and so we should expect to have to resort to relaxations. We took the first step to this end in Section \ref{sec:2ndorder} where we considered a second moment relaxation of non-discrimination which leads to tractable learning.  We hope this will not be the final word on learning non-discriminatory predictors and that this work will spur interest in further understanding our relaxation, suggesting other relaxations, and studying other computationally efficient procedures with provable guarantees.

\nocite{daniely2014average}
\clearpage
\bibliographystyle{plainnat}
\bibliography{two_step_lwod}
\clearpage
\appendix
\section{Deferred Proofs from Section~\ref{sec:suboptimalityposthoc}} \label{appendix:sec2}
\subsection{Proof of Example \ref{theorem:Step2LowerBound01}} \label{appendix:sec2-ex1}
We restate the example for convenience:
\exampleone*
Consider the unconstrained hypothesis class of all (possibly randomized) functions from $(X,A)$ to $\sset{0,1}$.
Let $\c{D}_\epsilon$ be the following distribution over $(X,A,Y)$, with $X,A,Y \in \sset{0,1}$:
\begin{equation} \label{eq:lowerbound01distr}
\bP(Y = 1) = 0.5 \qquad\qquad
\bP(A = y\ |\ Y = y) = 1 - \epsilon \qquad\qquad
\bP(X = y\ |\ Y = y) = 1 - 2\epsilon
\end{equation}
The graphical model representing this distribution is
\begin{center}
\begin{tikzpicture}
\node[shape=circle,draw=black] (X) at (0,0) {$X$};
\node[shape=circle,draw=black] (Y) at (1.5,0) {$Y$};
\node[shape=circle,draw=black] (A) at (3,0) {$A$};

\path [->] (Y) edge node[left] {} (X);
\path [->] (Y) edge node[left] {} (A);
\end{tikzpicture}
\end{center}
Clearly, $X \perp A\ |\ Y$, so $Y^*=X$ is non-discriminatory and achieves a 0-1 loss of $2\epsilon$. This same predictor achieves hinge loss $4\epsilon$. This predictor, being non-discriminatory, upper bounds the loss of the optimal non-discriminatory predictor with respect to the $0$-$1$ and hinge losses.

The optimal predictor with respect to the $0$-$1$ loss, which might be discriminatory, is in the convex hull (i.e.~it might be randomized combination) of the sized mappings from $\sset{0,1}\times\sset{0,1} \mapsto \sset{0,1}$. The Bayes optimal predictor with respect to the $0$-$1$ loss is the hypothesis
\begin{equation}
\hat{h}(x,a) = \argmax{y \in \sset{0,1}}\ \mathbb{P}(Y = y\ |\ X = x, A = a)
\end{equation}
Given $a \in \sset{0,1}$, note that since $\epsilon < 1/4$
\begin{equation}
\begin{aligned}
\mathbb{P}(Y = a\ |\ X = a, A = a) 
&= \frac{1}{1 + \frac{\mathbb{P}(A = a \ |\ Y = 1-a)\mathbb{P}(X = a \ |\ Y = 1-a)}{\mathbb{P}(A = a \ |\ Y = a)\mathbb{P}(X = a \ |\ Y = a)}} \\
&= \frac{1}{1 + \frac{(\epsilon)(2\epsilon)}{(1-\epsilon)(1-2\epsilon)}} > \frac{1}{2}.
\end{aligned}
\end{equation}
Similarly
\begin{equation}
\begin{aligned}
\mathbb{P}(Y = a\ |\ X = 1-a, A = a) 
&= \frac{1}{1 + \frac{\mathbb{P}(A = a \ |\ Y = 1-a)\mathbb{P}(X = 1-a \ |\ Y = 1-a)}{\mathbb{P}(A = a \ |\ Y = a)\mathbb{P}(X = 1-a \ |\ Y = a)}} \\
&= \frac{1}{1 + \frac{\epsilon(1-2\epsilon)}{(1-\epsilon)(2\epsilon)}} > \frac{1}{2}.
\end{aligned}
\end{equation}
Therefore, the Bayes optimal predictor is $\hat{h}(X,A) = A$, which is $1$-discriminatory, as 
\begin{equation}
\mathbb{P}(\hat{h}(X,A) = 1\ |\ Y = y, A = 0) = 0 \qquad\textrm{but}\qquad \mathbb{P}(\hat{h}(X,A) = 1\ |\ Y = y, A = 1) = 1
\end{equation}

Consider now the post-hoc correction $\tilde{h}$ of $\hat{h}$. The best non-discriminatory predictor $\tilde{Y}$ derived from the joint distribution $(\hat{h},A,Y) \equiv (A,A,Y)$ is given by the following optimization problem
\begin{equation}
\begin{aligned}
\tilde{Y} = \argmin{h}&\ \mathcal{L}(h) \\
\textrm{s.t.}\ \ & \gamma_{y0}(h) = \gamma_{y1}(h) && \text{for }y = 0,1 \\
& \begin{bmatrix} \gamma_{0a}(h) \\ \gamma_{1a}(h) \end{bmatrix} \in \textrm{ConvHull}\left(
\begin{bmatrix} 0 \\ 0 \end{bmatrix}, 
\begin{bmatrix} \gamma_{0a}(\hat{h}) \\ \gamma_{1a}(\hat{h}) \end{bmatrix}, 
\begin{bmatrix} \gamma_{0a}(1-\hat{h}) \\ \gamma_{1a}(1-\hat{h}) \end{bmatrix}, 
\begin{bmatrix} 1 \\ 1 \end{bmatrix}
\right)\quad &&\text{for }a = 0,1.
\end{aligned}
\end{equation}
The first constraint requires that the resultant predictor be non-discriminatory. The second requires that the class conditional true positive and false positive rates of the predictor be in the convex hull of the constant 0 predictor, the constant 1 predictor, the predictor $\hat{h}$ and its negative. This constraint is equivalent to requiring that $\tilde{h}$ be derived from $\hat{h}$. 
Because $\gamma_{y0}(\hat{h}) = 0$ and $\gamma_{y1}(\hat{h}) = 1$, the second constraint requires the predictor have equal true and false positive rates. As $\bP(Y = 1) = 0.5$, $0.5$ is optimal true and false positive rate and $\c{L}(\tilde{h})=0.5$.

Using the same distribution, but with $X,A,Y \in \sset{-1,1}$ instead of $\sset{0,1}$, the  optimal non-discriminatory predictor with respect to the hinge loss is no worse than the predictor which returns $X$, achieving hinge loss $4\epsilon$. However, the Bayes optimal predictor with respect to the hinge loss is again that predictor which returns 
\begin{equation*}
\hat{h}(x,a) = \argmax{y \in \sset{0,1}}\ \mathbb{P}(Y = y\ |\ X = x, A = a) = a.
\end{equation*}
By the same line of reasoning, the post hoc correction of $\hat{h}$, which must have identical statistics for $A = 0$ and $A = 1$ is forced to have equal true and false positive rates, and thus the best derived predictor is identically $0$ which has hinge loss 1.

\subsection{Proof of Example \ref{theorem:Step2LowerBoundConvex}}\label{appendix:sec2-ex2}
We restate the example for convenience:
\exampletwo*
In this example, we consider the squared loss and the hypothesis class of linear predictors with $L^1$ norm at most $\frac{1}{2} - 2\epsilon$ for some $\epsilon \in (2/25,1/4)$:
\[ \mathcal{H} = \sset{w_1X + w_2A + b : \abs{w_1} + \abs{w_2} \leq \frac{1}{2} - 2\epsilon}. \]
With this parameter $\epsilon$, we use the same distribution as in the proof of Theorem \ref{theorem:Step2LowerBound01}:
\begin{equation*}
\bP(Y = 1) = 0.5 \qquad\qquad
\bP(A = y\ |\ Y = y) = 1 - \epsilon \qquad\qquad
\bP(X = y\ |\ Y = y) = 1 - 2\epsilon
\end{equation*}
Since $X \independent A \ |\ Y$, any linear function of $X$ only will be non-discriminatory and $h(X) = \left(\frac{1}{2} - 2\epsilon\right)X + \frac{1}{4} + \epsilon$ has the required $L^1$ norm and achieves squared loss $\frac{1}{16} + \frac{3\epsilon}{2} + 3\epsilon^2$. This predictor, being non-discriminatory, upper bounds the squared loss of the optimal non-discriminatory predictor.

To derive the optimal potentially discriminatory predictor, it will be useful to begin by calculating the covariances between each of the variables:
\begin{align}
\mathbb{E}\left[ X \right] = \mathbb{E}\left[ A \right] = \mathbb{E}\left[ Y \right] &= \frac{1}{2} \\
\mathbb{E}\left[ X^2 \right] = \mathbb{E}\left[ A^2 \right] = \mathbb{E}\left[ Y^2 \right] &= \frac{1}{2} \\
\mathbb{E}\left[ XA \right] &= \frac{1}{2} - \frac{3}{2}\epsilon + 2\epsilon^2 \\
\mathbb{E}\left[ XY \right] &= \frac{1-2\epsilon}{2} \\
\mathbb{E}\left[ AY \right] &= \frac{1-\epsilon}{2} 
\end{align}
The optimal predictor optimizes
\begin{equation}
\begin{aligned}
\hat{h} = \argmin{w_1,w_2,b}&\ \mathbb{E}\left[ (w_1X + w_2A + b - y)^2\right] \\
\textrm{s.t.}&\quad \abs{w_1} + \abs{w_2} \leq \frac{1}{2} - 2\epsilon.
\end{aligned}
\end{equation}
Forming the Lagrangian:
\begin{align}
\mathcal{L}(w_1,w_2,b,\lambda) &= \mathbb{E}\left[ (w_1X + w_2A + b - y)^2\right] + \lambda\left( \abs{w_1} + \abs{w_2} - \frac{1}{2} - 2\epsilon \right) \nonumber\\
&= w_1^2 \mathbb{E}\left[ X^2 \right] + w_2^2 \mathbb{E}\left[ A^2 \right] + b^2 + \mathbb{E}\left[ Y^2 \right] + 2w_1w_2\mathbb{E}\left[ XA \right] + 2w_1b\mathbb{E}\left[ X \right] \nonumber\\
&- 2w_1\mathbb{E}\left[ XY \right] + 2w_2b\mathbb{E}\left[ A \right] -2w_2\mathbb{E}\left[ AY \right] -2b\mathbb{E}\left[ Y \right] + \lambda\left( \abs{w_1} + \abs{w_2} - \frac{1}{2} - 2\epsilon \right) \nonumber \nonumber\\
&= \frac{w_1^2 + w_2^2 + 1}{2} + b^2 + w_1w_2(1 - 3\epsilon + 4\epsilon^2) + w_1b - w_1(1-2\epsilon) + w_2b \nonumber\\
&- w_2(1-\epsilon) -b + \lambda\left( \abs{w_1} + \abs{w_2} - \frac{1}{2} - 2\epsilon \right).
\end{align}
At the following values:
\begin{equation}
w_1 = 0 \qquad\qquad w_2 = \frac{1}{2} - 2\epsilon \qquad\qquad b = \frac{1}{4} + \epsilon \qquad\qquad \lambda = \frac{1}{4}
\end{equation}
the subdifferential of $\mathcal{L}$ contains $0$ for any $\epsilon \in (2/25,1/4)$. Looking term by term we have:
\begin{align}
\frac{\partial \ell}{\partial w_1} &= w_1 + w_2(1 - 3\epsilon + 4\epsilon^2) + b - (1-2\epsilon) + \lambda\textrm{sign}(w_1) \\
&= \left(\frac{1}{2} - 2\epsilon\right)(1 - 3\epsilon + 4\epsilon^2) + \frac{1}{4} + \epsilon - (1-2\epsilon) + \frac{1}{4}\textrm{sign}(0) \\
&= \frac{1}{4}\textrm{sign}(0) - 8\epsilon^3 + 8\epsilon^2 - \frac{\epsilon}{2} - \frac{1}{4},
\end{align}
where $\textrm{sign}(0)$ is an arbitrary value in $[-1,1]$ which is the subdifferential of $f(z) = \abs{z}$ at $0$. For any $\epsilon \in (2/25,1/4)$ 
\begin{equation}
\abs{- 8\epsilon^3 + 8\epsilon^2 - \frac{\epsilon}{2} - \frac{1}{4}} < \frac{1}{4} \implies 0 \in \frac{\partial \ell}{\partial w_1}
\end{equation}
Furthermore,
\begin{equation}
\begin{split}
\frac{\partial \ell}{\partial w_2} &= w_1(1 - 3\epsilon + 4\epsilon^2) + w_2 + b - (1-\epsilon) + \lambda\textrm{sign}(w_2) \\
&= \frac{1}{2} - 2\epsilon + \frac{1}{4} + \epsilon - (1-\epsilon) + \frac{1}{4}= 0,
\end{split}
\end{equation}
and
\begin{equation}
\begin{split}
\frac{\partial \ell}{\partial b} &= 2b + w_1 + w_2 - 1 \\
&= \frac{1}{2} + 2\epsilon + \frac{1}{2} - 2\epsilon - 1 = 0.
\end{split}
\end{equation}

This proves that $\hat{h}(X,A) = \left( \frac{1}{2} - 2\epsilon \right)A + \frac{1}{4} + \epsilon$ is the Bayes optimal  predictor in $\c{H}$ with respect to the squared loss. Furthermore, the random variable is supported on only two points: $\frac{1}{4} + \epsilon$ when $A = 0$ and $\frac{3}{4} - \epsilon$ when $A = 1$. It is clear that $\hat{h}$ is not independent of $A$ conditioned on $Y$. 

Since $\hat{h}$ is a deterministic function of $A$, the post hoc correction $\tilde{h}$, which must be independent of $A$ conditioned on $Y$, is forced to be independent of $A$, and consequently $Y$. Thus, $\tilde{h} \equiv 0.5$ is the best possible derived predictor, achieving square loss $1/4$.

Considering the class of 1-sparse linear predictors, the predictor $h(X,A) = (1-4\epsilon)X + 2\epsilon$, being conditionally independent of $A$ is non-discriminatory and achieves squared loss $2\epsilon - 4\epsilon^2$, upper bounding the loss of the optimal non-discriminatory 1-sparse linear predictor. Without regard for non-discrimination, the optimal hypothesis in the class with respect to the squared loss is $\hat{h}(X,A) = (1-2\epsilon)A + \epsilon$. This post hoc correction of this predictor suffers from the same issue as in the bounded norm case, resulting in a predictor that can have squared loss no better than $1/4$.

\section{Deferred Proofs From Section \ref{sec:detection}}\label{appendix:detection}
Recall the notation $\Gamma_{ya}(h)=\max_y|\gamma_{y0}(h)-\gamma_{y1}(h)|$ and $\Gamma_{ya}^S(h)=\max_y|\gamma^S_{y0}(h)-\gamma^S_{y1}(h)|$ from Definition~\ref{def:approxnondiscrimination}. To avoid clutter, we sometimes drop the dependence on $h$ for $\gamma_{ya}$ when $h$ is evident from the context. Also, recall that $S=\{(x_i,y_i,a_i):i\in[n]\}\sim\b{P}^n(X,Y,A)$ and $\P_{ya}=\b{P}(Y=y,A=a)$. \\
\subsection{Proof of Lemma~\ref{lem:detectiontest}} 
\textit{\noindent \paragraph{Lemma \ref{lem:detectiontest}} Given $n$ i.i.d.~samples $S$, $\forall \alpha\in(0,1),\delta\in(0,1/2)$, if $n > \frac{16\log{32/\delta}}{\alpha^2\min_{ya}\P_{ya}}$, then  with probability greater than $1-\delta$, $T$ satisfies,
\[T\left(\hat{Y}, S,\frac{\alpha}{2}\right) = \begin{cases} 0 & \textrm{if } \hat{Y} \textrm{ is 0-discriminatory on population} \\ 1 & \textrm{if } \hat{Y} \textrm{ is at least } \alpha \textrm{-discriminatory on population.} \end{cases}
\]}

\begin{proof} 
Recall that ${T(\hat{Y}, S,\alpha) = \mathbf{1}\left( \Gamma^S(\hat{Y}) > \alpha \right)}$. Let $\alpha_n>0$ be any parameter chosen to satisfy, 
\begin{equation}2\max_{ya}\sqrt{\frac{\log{32/\delta}}{n\P_{ya}}}<\alpha_n<\alpha-2\max_{ya}\sqrt{\frac{\log{32/\delta}}{n\P_{ya}}}.\label{eq:alphn}
\end{equation}
Then the following results readily follow from Lemma~\ref{lemma:bin_step1},
\begin{compactenum}[1.]
\item If $\hat{Y}$ is non-discriminatory, i.e. $\Gamma(\hat{Y})=0$, then 
\begin{equation*}
\begin{aligned}
\mathbb{P}\left( T(\hat{Y}, S,\alpha_n)  = 1 \right)
&= \mathbb{P}( \Gamma^S(\hat{Y}) > \alpha_n ) \leq \mathbb{P}\bigg( \Gamma^S(\hat{Y})> \Gamma(\hat{Y})+ 2\max_{ya}\sqrt{\frac{\log{32/\delta}}{n\P_{ya}}}\bigg)\le\frac{\delta}{2}.
\end{aligned}
\end{equation*}
\item Similarly, suppose is $\hat{Y}$ is at least $\alpha$-discriminatory on the population, i.e.  $\Gamma(\hat{Y})\ge \alpha$, then 
\begin{equation*}
\begin{aligned}
\mathbb{P}\left( T(\hat{Y}, S,\alpha_n)  = 0 \right)
&= \mathbb{P}\left( \Gamma^S(\hat{Y}) \le \alpha_n \right) \leq \mathbb{P}\bigg(  \Gamma^S(\hat{Y})\le \Gamma(\hat{Y})- 2\max_{ya}\sqrt{\frac{\log{32/\delta}}{n\P_{ya}}}\bigg)
\leq \frac{\delta}{2}.
\end{aligned}
\end{equation*}
\end{compactenum}
Thus, if $\frac{\alpha}{2}>  2\max_{ya}\sqrt{\frac{\log{32/\delta}}{n\P_{ya}}}$, then $\alpha_n=\frac{\alpha}{2}$ satisfies \eqref{eq:alphn} and  Lemma~\ref{lem:detectiontest} follows for $T(\hat{Y}, S,\frac{\alpha}{2})$.
\end{proof}

\subsection{Proof of Lemma~\ref{lemma:bin_step1}}
\textit{\noindent \paragraph{Lemma \ref{lemma:bin_step1}}  For $\delta\in(0,1/2)$ and a binary predictor $h$, if $n>\frac{8\log{8/\delta}}{\min_{ya}\P_{ya}}$, then
\begin{equation}
\begin{split}
&\b{P}\left(\left|\Gamma(h)-\Gamma^S(h)\right|> 2\max_{ya}\sqrt{\frac{\log{16/\delta}}{n\P_{ya}}}\right)\le \delta.
\end{split}
\end{equation} }
\begin{proof}
Recall that  
$n^S_{ya}=\sum_i\mathbf{1}({y_i=y,a_i=a})$. With slight abuse of notation, we define random variables  $S_{ya}=\{i:y_i=y,a_i=a\}$. \\
We then have $\gamma_{ya}^S(h)|S_{ya}=\frac{{\sum_{j\in S_{ya}}h(x_j,a_j)}}{n^S_{ya}}\sim \frac{1}{n_{ya}^S} \text{Binomial}(\gamma_{ya},n^S_{ya})$ with $\b{E}[\gamma_{ya}^S|S_{ya}]=\gamma_{ya}$.
\begin{flalign}
\nonumber\bP\left(|\gamma_{ya}^S-\gamma_{ya}|>t\right)&\overset{(a)}=\sum_{S_{ya}}\bP\left(|\gamma_{ya}^S-\gamma_{ya}|>t|S_{ya}\right)\bP\left(S_{ya}\right) &\\
\nonumber &\le \bP\left(n^S_{ya}<\frac{n\P_{ya}} 2\right)+\sum_{S_{ya}: n^S_{ya}\ge\frac{n\P_{ya}}2}\bP\left(|\gamma_{ya}^S-\gamma_{ya}|>t|S_{ya}\right)\bP\left(S_{ya}\right)&\\
\nonumber &\overset{(b)}\le  \exp{\big(- \frac{n\P_{ya}}8\big)}+\sum_{S_{ya}:n^S_{ya}\ge\frac{n\P_{ya}}2}2\exp{(-2t^2n^S_{ya})}\bP(S_{ya}) &\\
&\overset{(c)}{\le} \frac{\delta}{8}+2\exp{(-t^2n\P_{ya})},&
\label{eq:conc_cond_prob}
\end{flalign}
where in $(a)$ the summation is over all $2^n$ possible configurations of $S_{ya}\subset[n]$,  $(b)$ follow from Chernoff bound on $n^S_{ya}\sim\text{Binomial}(n,\P_{ya})$ and Hoeffding's bound on $\gamma_{ya}^S|S_{ya}\sim \frac{1}{n^S_{ya}}\text{Binomial}(n^S_{ya},\gamma_{ya})$, and $(c)$ follows from the condition on $n\P_{ya}$ in Lemma~\ref{lemma:bin_step1}.

Further, for $y\in\{0,1\}$, using a series of triangle inequality, $$\left||\gamma^S_{y0}-\gamma^S_{y1}|-|\gamma_{y0}-\gamma_{y1}|\right|\le|\gamma^S_{y0}-\gamma^S_{y1}-\gamma_{y0}+\gamma_{y1}|\le|\gamma^S_{y0}-\gamma^S_{y0}|+|\gamma^S_{y1}-\gamma_{y1}|,\text{ and }$$
\begin{flalign}
\nonumber\bP_S\left(\left||\gamma^S_{y0}-\gamma^S_{y1}|-|\gamma_{y0}-\gamma_{y1}|\right|> 2t\right) &
\le \bP_S\left(|\gamma^S_{y0}-\gamma_{y0}|+|\gamma^S_{y1}-\gamma_{y1}|>  2t\right)&\\
 \nonumber &\overset{(a)}\le\bP_S\Big(|\gamma^S_{y0}-\gamma_{y0}|> t\Big)+\bP_S\Big(|\gamma^S_{y1}-\gamma_{y1}|> t\Big)&\\
&\le\frac{\delta}{4}+4\exp\left(-t^2n\P_{ya}\right)\overset{(b)}\le \frac{\delta}{2},&
\end{flalign}
where $(a)$ follows from union bound, and $(b)$ follows from \eqref{eq:conc_cond_prob} using $t=\max_{a}\sqrt{\frac{\log{16/\delta}}{n\P_{ya}}}$. 
The lemma follows from collecting the failure probabilities for $y=0,1$.
\end{proof}

\section{Deferred Proofs From Section \ref{sec:binary}}\label{app:binary}
We use the notation $A\le _\delta B$ to denote that $A\le B$ holds with probability greater than $1-\delta$. Recall the notation $\Gamma_{ya}(h)=\max_y|\gamma_{y0}(h)-\gamma_{y1}(h)|$ and $\Gamma_{ya}^S(h)=\max_y|\gamma^S_{y0}(h)-\gamma^S_{y1}(h)|$ from Definition~\ref{def:approxnondiscrimination}. To avoid clutter, we sometimes drop the dependence on $h$ for $\gamma_{ya}$ when $h$ is evident from the context. Finally, in this section $C,C_1$ and $C_2$  denote absolute constants that are not necessarily the same at each occurrence.
\subsection{Proof of Lemma~\ref{lemma:BinaryStep1Guarantee}}\label{app:4_def}
The following intermediate lemma is used in proof of Lemma~\ref{lemma:BinaryStep1Guarantee}. 
\begin{lemma}\label{lemma:Heps}
Let  $\c{H}^{S_1}_{\alpha_{n}}=\{h\in\c{H}: \Gamma^{S_1}(h)\le \alpha_{n}\}$ denote the subset of hypothesis that satisfy the constraints in \eqref{eq:step1}.
If \,$\forall (y,a),\,n\P_{ya}>16\log{8/\delta}$ and $\alpha_{n}$ in \eqref{eq:step1} satisfies $\alpha_{n}\ge2\max_{ya}\sqrt{\frac{2\log{64/\delta}}{n\P_{ya}}}$, then  with probability  greater than $1-\frac{\delta}{4}$, $\Yfair\in\c{H}_{\alpha_n}^{S_1}$ for all $\Yfair\in\c{Q}(\c{L}(Y^*),0)\cap \c{H}$.
\end{lemma}
\begin{proof}   Given $\Yfair \in\c{H}\cap \c{Q}(L^*,0)$, 
\begin{flalign}
\nonumber \b{P}(\Yfair\notin\c{H}_{\alpha_n}^{S_1})&=\b{P}(\Gamma^{S_1}(\Yfair)>\alpha_n)
\overset{(a)}\le \b{P}\Big(\Gamma^{S_1}(\Yfair)>2\max_{ya}\sqrt{\frac{2\log{64/\delta}}{n\P_{ya}}}\Big)&\\
&\overset{(b)}\le \b{P}\Big(\Gamma^{S_1}(\Yfair)>\Gamma(\Yfair)+2\max_{ya}\sqrt{\frac{2\log{64/\delta}}{n\P_{ya}}}\Big)&\\
&\overset{(c)}{\le} \frac{\delta}{4},
\end{flalign}
where $(a)$ follows from the condition on $\alpha_n$, $(b)$ follows from $\Gamma(Y^*)\!=0$ assumption, and $(c)$ follows from Lemma~\ref{lemma:bin_step1} as  $|S_1|=n/2$. 
\end{proof}

\textit{\noindent \paragraph{Lemma~\ref{lemma:BinaryStep1Guarantee}} 
Under the conditions in Theorem~\ref{theorem:UpperBounds}, if $\alpha_n>2\max_{ya}\sqrt{\frac{2\log{64/\delta}}{n\P_{ya}}}$, then w.p. greater than $1-\delta$, $\hat{Y}$ from Step $1$ satisfies 
$$\c{L}(\hat{Y}) \leq \c{L}(Y^*) + C_1\Da \text{, and }\Gamma(\hat{Y}) \le \alpha_n+ C_2\epsa.$$}

\begin{proof}
Recall that the training data for Step $1$ are denoted by  $S_1=\{(x_i,a_i,y_i):i\in[n/2]\}\sim\b{P}^{n/2}(X,A,Y)$ and $\P_{ya}=\b{P}(Y=y,A=a)$.  From Using Hoeffding's inequality  on the empirical $0$-$1$ loss $\c{L}^{S_1}(h)$, and using concentration results for $\Gamma^{S_1}(h)$ from Lemma~\ref{lemma:bin_step1}, respectively, the following holds for $\delta\in (0,1/2)$ and $\min_{ya}n\P_{ya}>16\log{8/\delta}$.
\begin{equation}
\big|\c{L}(h)-\c{L}^{S_1}(h)\big| \le_{\delta/4} \sqrt{\frac{\log{8/\delta}}{n}},\text{ and} \quad
|\Gamma(h)-\Gamma^{S_1}(h)|\le_{\delta/4} 2\max_{ya}\sqrt{\frac{2\log{64/\delta}}{n\P_{ya}}}.
\label{eq:conc}
\end{equation} 

Using \eqref{eq:conc} and the standard VC dimension uniform bound \citep{bousquet2004introduction},  the following holds with high probability for absolute constants $C_1$ and $C_2$, 
\begin{equation}
\begin{split}
 |\c{L}(\hat{Y})-\c{L}^{S_1}(\hat{Y})|&\le_{\delta/4} C_1\sqrt{\frac{VC(\c{H})+\log{1/\delta}}{n}} \text{,\quad and  }\\
|\Gamma(\hat{Y})-\Gamma^{S_1}(\hat{Y})|&\le_{\delta/4} C_2\max_{ya}\sqrt{\frac{VC(\c{H})+\log{1/\delta}}{n\P_{ya}}}.
 \end{split}
 \label{eq:gamma_conc}
 \end{equation}
Finally, from Lemma~\ref{lemma:Heps}, with probability greater than $1-\delta/4$, any $0$-discriminatory $\Yfair\in\c{H}$ is in the feasible set for Step $1$ in \eqref{eq:step1},  and thus from the optimality of $\hat{Y}$,  $\c{L}^{S_1}(\hat{Y})\le_{\delta/4}\c{L}^{S_1}(\Yfair)\le_{\delta/4} \c{L}(Y^*)+ C_1\sqrt{\frac{VC(\c{H})+\log{1/\delta}}{n}}$. Thus, 
\begin{flalign}
\nonumber&\c{L}(\hat{Y})\le_{\delta/4} \c{L}^{S_1}(\hat{Y}) +C_1\sqrt{\frac{VC(\c{H})+\log{1/\delta}}{n}}\le_{\delta/2} \c{L}(Y^*)+ 2C_1\sqrt{\frac{VC(\c{H})+\log{1/\delta}}{n}},&\\
&\Gamma(\hat{Y})\le_{{\delta}/4} \alpha_n +C_2\max_{ya}\sqrt{\frac{VC(\c{H})+\log{1/\delta}}{n\P_{ya}}}.&\
\label{eq:tmp1}
\end{flalign}
The lemma follows from  combining the failure probabilities in the above equation.
\end{proof}

\subsection{Proof of Lemma~\ref{steptwotopop}}
\textit{\noindent \paragraph{Lemma~\ref{steptwotopop}}
If $h$ is an $\alpha$-discriminatory binary predictor $h\in\c{Q}(\c{L}(h),\alpha)$, then the optimal 0-discriminatory derived predictor $\tilde{Y}^*(h)$  from \eqref{eq:post-hocalgorithm} using 0-1 loss satisfies $\tilde{Y}^*(h)\in\c{Q}(\c{L}(h)+\alpha,0)$.}

\begin{proof} 
The intuition is to conservatively bound the true and false positive rates of the non-discriminatory derived predictor using the class conditional rates for $h$. In the case of binary predictors, $\tilde{Y}$ being derived from $h$ is equivalent to requiring that 
\begin{equation}
\big(\gamma_{0a}(\tilde{Y}(h)),\gamma_{1a}(\tilde{Y}(h))\big) \in  \textrm{Conv}\left( (0,0), (1,1), (\gamma_{0a}(h), \gamma_{1a}(h)), (1 - \gamma_{0a}(h), 1 - \gamma_{1a}(h)) \right)
\label{eq:convhull}
\end{equation}
In the figure below, $\tilde{Y}^*(h)$ is the actual optimal derived non-discriminatory predictor, but we estimate it conservatively using the worse of the class conditional true and false positive rates of $h$:
\begin{figure}[H]
\begin{center}
\includegraphics[width=6.5cm]{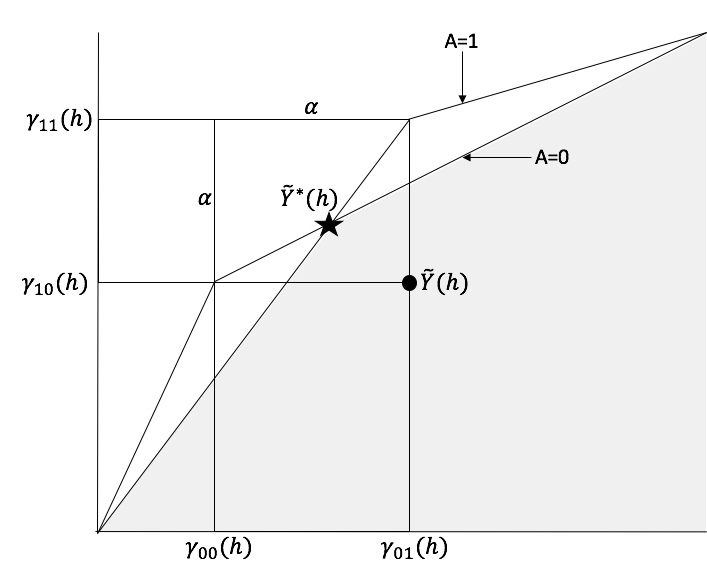}
\end{center}
\end{figure} 
Without loss of generality, assume $\gamma_{1a}(h) \geq 0.5$ and $\gamma_{0a}(h) \leq 0.5$ for all $a$ (the hypothesis  is at least as good as chance).  Consider the predictor $\tilde{Y}$ such that for both $a\in\{0,1\},$
\[ \big(\gamma_{0a}(\tilde{Y}),\gamma_{1a}(\tilde{Y})\big)= \left( \max(\gamma_{00}, \gamma_{01}),\min(\gamma_{10}, \gamma_{11}) \right)\in\textrm{Conv}\left( (0,0), (1,1), (\gamma_{0a}(h), \gamma_{1a}(h))\right) \] 
that is, $\tilde{Y}$ has the greater of the two false positive rates and lesser of the two true positive rates for both classes $A=1$ and $A = 0$. 
Additionally,  Lemma~\ref{steptwotopop} requires that $\forall y$, $|\gamma_{y1}(h)-\gamma_{y0}(h)|\le\alpha$. Thus, for $a\in\{0,1\}$, 
\begin{equation}
\begin{split}
 \gamma_{0a}(\tilde{Y}) - \gamma_{0a}(h) &=\max_{a'}\gamma_{0a'}(h) - \gamma_{0a}(h)  \le \alpha, \text{ and }\\
\gamma_{1a}(h) - \gamma_{1a}(\tilde{Y})&=\gamma_{1a}(h) - \min_{a'}\gamma_{1a'}(h)  \le \alpha. 
\end{split}
\end{equation}

Clearly, this choice of $\tilde{Y}$ is both non-discriminatory (as $\gamma_{ya}(\tilde{Y})$ is set independent of $a$ for all $y$), as well as derived (as $\gamma_{ya}(\tilde{Y})$ satisfy \eqref{eq:convhull}). Thus $\tilde{Y}$ is a feasible point for \eqref{eq:post-hocalgorithm}, and we have
\begin{equation}
\begin{aligned}\label{eq:proof18eq}
\mathbb{E}[\ell^{01}(\tilde{Y}^*(h))]&\le \mathbb{E}[\ell^{01}(\tilde{Y})] = \sum_{a\in\sset{0,1}} \P_{0a}{\gamma}_{0a}(\tilde{Y}) + \sum_{a\in\sset{0,1}}\P_{1a}(1-{\gamma}_{1a}(\tilde{Y}))\\
&\overset{(a)}\le\sum_{a\in\sset{0,1}}\P_{0a}\left(\gamma_{0a}(h)+\alpha\right) + \sum_{a\in\sset{0,1}}\P_{1a}\left(1-(\gamma_{1a}(h) - \alpha)\right) \\
&\leq \mathbb{E}[\ell^{01}(h)] +\alpha,
\end{aligned}
\end{equation}
where $(a)$ follows from \eqref{eq:proof18eq}.
\end{proof}

\subsection{Proof of Theorem~\ref{theorem:UpperBounds}}\label{app:ub}
The following supporting lemma on concentration of non-discrimination for randomized predictors is used in the proof of Theorem~\ref{theorem:UpperBounds}.
\begin{lemma} For $\delta\in(0,1/2)$ and $h\in\c{H}$, if $n\P_{ya}>16\log{8/\delta}$, For any randomized predictor $\tilde{Y}$ derived from $(\hat{Y},A)$, i.e.$\tilde{Y}\in\c{P}(\hat{Y})$, satisfies the following for $|S_2|=n/2$ iid sampels:
\begin{equation}
\big|\c{L}(\tilde{Y})-\c{L}^{S_2}(\tilde{Y})\big| \le_{\delta} \sqrt{\frac{\log{2/\delta}}{n}},\text{ and } 
|\Gamma(\tilde{Y})-\Gamma^{S_2}(\tilde{Y})|\le_{\delta} 2\max_{ya}\sqrt{\frac{2\log{16/\delta}}{n\P_{ya}}}.
\end{equation} 
\label{lemma:bin_step2}
\end{lemma}
\begin{proof} The proof essentially follows the same arguments that were used  for \eqref{eq:conc} and Lemma~\ref{lemma:bin_step1}.

For a randomized predictor $\tilde{Y}$,  $$\c{L}^{S_2}(\tilde{Y})-\c{L}(\tilde{Y})=\frac{2}{n}\sum_{i\in S_2} \b{E}_{\tilde{Y}} \ell(\tilde{Y}(\hat{y}_i,a_i),y_i)-\b{E}_{X,A,Y}\b{E}_{\tilde{Y}}\ell(\tilde{Y},Y)$$
Here $\c{L}^{S_2}(\tilde{Y})-\c{L}(\tilde{Y})$ is merely a sum of $n/2$ independent and $[0,1]$ bounded random variables and Hoeffdings bound can be applied to get the required concentration  on $\c{L}^{S_2}$.

Similarly, for any randomized predictor $\tilde{Y}$, the conditional random variable $\gamma_{ya}^{S_2}(\tilde{Y})|S_{2_{ya}}=\frac{{\sum_{j\in S_{2_{ya}}}\b{E}_{\tilde{Y}}\tilde{Y}(\hat{y}_j,a_j)}}{n^{S_2}_{ya}}$ is a sum of $[0,1]$ bounded random variables with mean $\b{E}[\gamma_{ya}^{S_2}(\tilde{Y})|S_{2_{ya}}]=\gamma_{ya}(\tilde{Y})$, the proof of Lemma~\ref{lemma:bin_step1} can be repeated verbatim for the randomized prediction where instead of the Hoeffdings' bound on Binomial random variables, we use the identical Hoeffdings' bound for $[0,1]$ bounded random variables.
\end{proof}

\thmub*
\begin{proof} We begin with the following result from from Lemma~\ref{lemma:bin_step2} which shows the concentration of loss and discrimination in radomized derived predictors: if $n\P_{ya}>16\log{8/\delta}$, for any randomized predictor $\tilde{h}$ derived from $(\hat{Y},A)$, i.e. $\tilde{h}\in\c{P}(\hat{Y})$, satisfies the following:
\begin{equation}
\big|\c{L}(\tilde{h})-\c{L}^{S_2}(\tilde{h})\big| \le_{\delta/4} \sqrt{\frac{\log{8/\delta}}{n}},\text{ and } 
|\Gamma(\tilde{h})-\Gamma^{S_2}(\tilde{h})|\le_{\delta/4} 2\max_{ya}\sqrt{\frac{2\log{64/\delta}}{n\P_{ya}}}.
\label{eq:conc2}
\end{equation}

\paragraph{VC dimension of $\c{P}(\hat{Y})$: }
If $|\hat{\c{Y}}|$ and $|\c{A}|$ are finite, consider a finite  hypothesis class denoted by $\c{H}_{\hat{Y},A}$, that includes all deterministic mappings from $(\hat{Y},A)$ to binary values $\{0,1\}$.  If $\hat{Y}$ and ${A}$ are both binary, then there are $4^2=16$ such mappings  $\c{H}_{\hat{Y},A}=\{\tilde{h}:\{0,1\}\times \{0,1\}\to \{0,1\}\}$. 

Further, recall that  the feasible set of (randomized) binary predictors derived from $\bP(\hat{Y},Y,A)$ is denoted by $\c{P}(\hat{Y})$, and  any $\tilde{Y}\in\c{P}(\hat{Y})$ is completely specified by four parameters, $\{\tilde{p}_{\hat{y},a}(\tilde{Y})=\bP(\tilde{Y}=1|\hat{Y}=\hat{y},A=a):\hat{y},a\in\{0,1\}\}$. With the above definition of $\c{H}_{\hat{Y},A}$, any such randomized derived predictor $\tilde{Y}\in \c{P}(\hat{Y})$  derived from $(\hat{Y},{A})$ is in the convex hull of $\c{H}_{\hat{Y},A}$.  This implies, $VC(\c{P}(\hat{Y}))=VC(\text{conv}(\c{H}_{\hat{Y},A}))=\log{16}$ which is a constant. 

Thus, for $\forall \tilde{Y}\in\c{P}(\hat{Y})$ estimated from Step $2$ in \eqref{eq:step2},  using the standard VC dimension uniform bound over $\c{P}(\hat{Y})$ \citep{bousquet2004introduction} along with \eqref{eq:conc2} we have the following, 
\begin{flalign}
\c{L}(\tilde{Y})\le_{\delta/4} \c{L}^{S_2}(\tilde{Y}) +C_1\sqrt{\frac{\log{1/\delta}}{n}}
\text{, and  }\Gamma(\tilde{Y})\le_{\delta/4} \tilde{\alpha}_n +C_2\max_{ya}\sqrt{\frac{\log{1/\delta}}{n\P_{ya}}}.
\label{eq:ubstep2_1}
\end{flalign}

\paragraph{Upper bound on $\c{L}^{S_2}(\tilde{Y})$: }
For any derived  $\tilde{Y}^*\in\c{P}(\hat{Y})$ that is $0$-discriminatory, using \eqref{eq:conc2} and the identical arguments as that of Lemma~\ref{lemma:Heps}, we have the following:
\begin{flalign*}
 \b{P}(\Gamma^{S_2}(\tilde{Y}^*)>\tilde{\alpha}_n)
\overset{(a)}\le \b{P}\Big(\Gamma^{S_2}(\tilde{Y}^*)>2\max_{ya}\sqrt{\frac{2\log{64/\delta}}{n\P_{ya}}}\Big)\overset{(b)}{\le} \frac{\delta}{4},
\end{flalign*}
where $(a)$ follows from the condition on $\tilde{\alpha}_n$ and $(b)$ from Lemma~\ref{lemma:bin_step2}. 

From the optimality of $\c{L}^{S_2}(\tilde{Y})$ and Lemma~\ref{lemma:bin_step2}, 
$\forall \tilde{Y}^*\in\c{Q}(\c{L}(\tilde{Y}^*),0)\cap \c{P}(\hat{Y})$ we have,
\begin{equation}
\c{L}^{S_2}(\tilde{Y})\le_{\delta/4} \c{L}^{S_2}(\tilde{Y}^*)\le_{\delta/4} \c{L}(\tilde{Y}^*)+C_1\sqrt{\frac{\log{8/\delta}}{n}}.
\label{eq:upstep2}
\end{equation}

\paragraph{Upper bound on $\c{L}(\tilde{Y}^*)$: }
 The rest of the proof involves obtaining an upper bound for $\c{L}(\tilde{Y}^*)$ using Lemma~\ref{steptwotopop}. 
 Recall from Lemma~\ref{steptwotopop} that if $h$ is an $\alpha$-discriminatory binary predictor, the optimum non-discriminatory derived predictor $\tilde{Y}^*(h)$ given by \eqref{eq:post-hocalgorithm} satisfies, $\tilde{Y}^*(h)\in\c{Q}(\c{L}(h)+\alpha,0)$.

 Thus,   for $h=\hat{Y}$, a derived non-discriminatory predictor $\tilde{Y}^*(\hat{Y})$ obtained from  \eqref{eq:post-hocalgorithm} satisfies, 
\begin{equation}
\begin{split}
\c{L}(\tilde{Y}^*(\hat{Y}))\overset{(a)}\le_\delta \c{L}(Y^*)+\alpha_n+C_3\epsa,
\end{split}
\label{eq:upstep2_3}
\end{equation}
 where in $(a)$ $Y^*\in\c{H}$ is any non-discriminatory predictor from the original hypothesis class $\c{H}$ and the inequality follows 
  from combining Lemma~\ref{lemma:BinaryStep1Guarantee} and Lemma~\ref{steptwotopop} as with probability atleast $1-\delta$,   $\hat{Y}$ is atmost $\alpha=\Gamma(\hat{Y}) \le\alpha_n+ C_2\epsa$ discriminatory.
  
  Combining  \eqref{eq:ubstep2_1}, \eqref{eq:upstep2}, and \eqref{eq:upstep2_3}, with probability atleast $1-2\delta$,
  \begin{equation}
  \begin{split}
  \c{L}(\tilde{Y})&\le\c{L}(Y^*)+\alpha_n+C_1\epsa,\text{ and }\\
  \Gamma(\tilde{Y})&\le\tilde{\alpha}_n +C_2\max_{ya}\sqrt{\frac{\log{1/\delta}}{n\P_{ya}}}
  \end{split}
  \end{equation}
  Theorem~\ref{theorem:UpperBounds} follows from appropriate 
choice of $\alpha_n,\tilde{\alpha}_n$  and rescaling $\delta$. 
  \end{proof}

\subsection{Proof of Theorem \ref{theorem:Step1LowerBound}}
Let the marginal distribution over $(A,Y)$ be given by $p = \min_{a,y} \mathbb{P}(A = a, Y = y)$. Since the definiton of fairness is invariant to re-labelling of $A,Y$, assume without loss of generality that $p$ corresponds to $A = 1, Y = 1$. For $\alpha \in (0,1/2)$, the distribution $\mathcal{D}$ over $(X,A,Y) \in \sset{0,1}^n \times \sset{0,1} \times \sset{0,1}$ is described by
\begin{equation}
\begin{aligned}
&\bP(X_1 = y\ |\ Y = y) &&= 1 - \alpha \\
&\bP(X_i = 0\ |\ Y = 0, A = 0) &&= 1 && \textrm{ for } i = 2,3,...,n \\
&\bP(X_i = 1\ |\ Y = 1, A = 0) &&= 1 && \textrm{ for } i = 2,3,...,n \\
&\bP(X_i = 0\ |\ Y = 0, A = 1) &&= 1 && \textrm{ for } i = 2,3,...,n \\
&\bP(X_i = 1\ |\ Y = 1, A = 1) &&= 1 - \alpha && \textrm{ for } i = 2,3,...,n \\
\end{aligned}
\end{equation}
Consider the following hypothesis class $\mathcal{H} = \sset{h_i}_{i=1}^n$ with $h_i(X,A) = X_i$.
The hypothesis $h_1$ has 0-1 loss $\c{L}_{01}(h_1) = \bP(X_1 \neq Y) = \alpha$ and is exactly non-discriminatory since $X_1 \perp A\ |\ Y$ by construction. For every other $i = 2,3,...,n$, the 0-1 loss of $h_i$ is the same: 
\begin{equation}
\c{L}_{01}(h_i) = \sum_{y}\sum_{a} \bP(X_i = 1 - y\ |\ Y = y, A = a)\bP(Y = y, A = a) = p\alpha
\end{equation}
however, for these hypotheses $\abs{\bP(h_i = 1 | Y = 1, A = 1) - \bP(h_i = 1 | Y = 1, A = 0)} = \alpha$ so $h_i$ is $\alpha$-discriminatory. 

We will now show that on a sample $S$ of size $m$, the empirical risk minimizer subject to an approximate non-discrimination constraint, $\hat{h}$, will be $h_i$ for $i \neq 1$ with probability $0.5$. Hence, the first step alone cannot assure with probability better than $0.5$ a classifier that is better than $\alpha$-discriminatory.

First, we note that the predictions of $h_i$ and $h_j$ are independent for $i \neq j$ since $X_i$ and $X_j$ are independent. Therefore, the number of errors made by each classifier $h_i$ on the sample $S$ are independent and
\begin{equation}
\begin{aligned}
\mathbb{P}(\mathcal{L}_{01}^S(h_1) = 0) = (1-\alpha)^m \qquad\qquad \bP(\c{L}_{01}^S(h_i) = 0) = (1-p\alpha)^m \quad\textrm{for } i = 2,3,...,n
\end{aligned}
\end{equation}
Since a classifier that makes zero errors on $S$ is automatically non-discriminatory on $S$, if $h_1$ makes at least $1$ mistake on $S$ and some $h_i$ does not make any errors, then $h_1$ will be the optimum of \eqref{eq:step1}. This event occurs with probability:
\begin{align}
\bP\left(\c{L}_{01}^S(h_1) > 0 \wedge \exists i\ \c{L}_{01}^S(h_i) = 0\right) &= \left(1 - (1-\alpha)^m\right)\left(1 - \bP\left(\forall i > 1\ \c{L}_{01}^S(h_i) > 0\right) \right) \\
&= \left(1 - (1-\alpha)^m\right)\left( 1 - \prod_{i=2}^n \left(1 - (1-p\alpha)^m\right) \right) \\
&= \left(1 - (1-\alpha)^m\right)\left(1 - \left(1 - (1-p\alpha)^m\right)^{n-1}\right)
\end{align}
From here, we use that
\begin{align}
\forall k \in \mathbb{N}\ \forall x \in [0,1]\quad (1-x)^k \leq \frac{1}{1+kx}
\end{align}
Thus, since $1-p\alpha,\alpha \in [0,1]$:
\begin{align}
\bP\left(\c{L}_{01}^S(h_1) > 0 \wedge \exists i\ \c{L}_{01}^S(h_i) = 0\right)
&= \left(1 - (1-\alpha)^m\right)\left(1 - \left(1 - (1-p\alpha)^m\right)^{n-1}\right) \\ 
&\geq \left( 1 - \frac{1}{1+m\alpha}\right)\left(1 - \frac{1}{1 + (n-1)(1-p\alpha)^m} \right) \\
&= \frac{m\alpha}{1+m\alpha} - \left( 1 - \frac{1}{1+m\alpha}\right)\frac{1}{1 + (n-1)(1-p\alpha)^m} \label{eq:step1lowerboundtwoterms}
\end{align}
This expression is is greater than $1/2$ if the first term is at least $2/3$ and the second term is at most $1/6$. Thus
\begin{equation}
\frac{m\alpha}{1+m\alpha} \geq \frac{2}{3} \iff \alpha \geq \frac{2}{m}
\end{equation}
and
\begin{equation}
\begin{aligned}
&\left( 1 - \frac{1}{1+m\alpha}\right)\frac{1}{1 + (n-1)(1-p\alpha)^m} \leq \frac{1}{6} \\
&\impliedby \frac{1}{1 + (n-1)(1-p\alpha)^m} \leq \frac{1}{6} \iff \log\frac{1}{1-p\alpha} \leq \frac{\log \frac{n-1}{5}}{m}
\end{aligned}
\end{equation}
Since $-\log(1-x) < \frac{x}{1-x}$ for $x \in (0,1)$, and $p \leq \frac{1}{4}$, the expression \eqref{eq:step1lowerboundtwoterms} is at least $1/2$ when 
\begin{equation}
\frac{2}{m} \leq \alpha \leq \frac{3\log \frac{n}{5}}{4pm}
\end{equation}

Therefore, when $\alpha = \frac{3\log \frac{n}{5}}{4pm}$, with probability $0.5$ $h_1$ has non-zero error on $S$ and a different predictor has zero error.
We conclude that there exists a distribution and hypothesis class such that with probability $0.5$, the hypothesis returned by the first step is $\frac{3\log \frac{n}{5}}{4pm}$-discriminatory.

\section{Proof of Theorem \ref{thm:hardness}} \label{appendix:hardness}
Let $\mathcal{A}$ be an algorithm that takes as inputs a hypothesis class $\mathcal{H}$, a distribution $\tilde{\mathcal{D}}$ over $(\tilde{X}, \tilde{A}, \tilde{Y})$ with $\tilde{A} \in \sset{0,1}$ and $\tilde{Y} \in \sset{-1,+1}$, an accuracy parameter $\epsilon > 0$, and a non-discrimination parameter $\alpha > 0$ and returns a predictor $f = \mathcal{A}(\tilde{\mathcal{D}}, \epsilon, \alpha)$ such that with probability $1-\zeta$
\begin{gather}
\mathcal{L}_{\tilde{\mathcal{D}}}^{\textrm{hinge}}\left(f\right) \leq \min_{h \in \mathcal{H}_{\textrm{0-disc}}(\tilde{\mathcal{D}})} \mathcal{L}_{\tilde{\mathcal{D}}}^{\textrm{hinge}}(h) + \epsilon \label{eq:definitionAlg}\\
\abs{\mathbb{P}_{\tilde{\mathcal{D}}}\left(f \geq 0\ \middle|\ \tilde{Y} = y, \tilde{A} = 0 \right) - \mathbb{P}_{\tilde{\mathcal{D}}}\left(f \geq 0\ \middle|\ \tilde{Y} = y, \tilde{A} = 1 \right)} \leq \alpha\qquad\textrm{for } y = -1,+1 \nonumber
\end{gather}
The possibly randomized predictor $f$ need not be in the hypothesis class $\mathcal{H}$, but it is being compared against the best predictor in $\mathcal{H}$ whose sign is non-discriminatory. 

We will show that such an algorithm can be used to improperly weakly learn \textsc{Halfspace} which, subject to the complexity assumption that refuting random K-XOR formulas is hard, was shown to be computationally hard by \cite{daniely2015complexity}. We conclude that $\mathcal{A}$ must be computationally hard to compute.

The \textsc{Halfspace} problem is to take a distribution $\mathcal{D}$ over $(X,Y)$ with $X \in \mathbb{R}^d$ and $Y \in \sset{-1,+1}$, and find the linear predictor 
\begin{equation} \label{eq:halfspacehstar} 
h^*(x) = \textrm{sign}({w^*}^Tx) \qquad\textrm{where}\qquad w^* = \argmin{w \in \mathbb{R}^d}\ \underset{x,y \sim \mathcal{D}}{\mathbb{E}}\left[ \textrm{sign}(w^Tx) \neq y \right]
\end{equation}
The proof of hardenss of the \textsc{Halfspace} problem was shown using a distribution over the unit hypercube in $d$-dimensions, thus we will assume that $\mathcal{D}$ is a bounded distribution. We assume access to the distribution $\mathcal{D}$, knowledge of $\mathcal{L}_{\mathcal{D}}^{01}(h^*)$, and for now, access to the joint distribution of $(h^*(X), Y)$:
\begin{equation}
\begin{aligned}
&\eta_{--} = \mathbb{P}_{\mathcal{D}}(h^*(X) = -1, Y = -1)
\qquad&&\eta_{-+} = \mathbb{P}_{\mathcal{D}}(h^*(X) = -1, Y = +1) \\
&\eta_{+-} = \mathbb{P}_{\mathcal{D}}(h^*(X) = +1, Y = -1) 
\qquad&&\eta_{++} = \mathbb{P}_{\mathcal{D}}(h^*(X) = +1, Y = +1)
\end{aligned}
\end{equation}
however, we will show later that it is not necessary to know the $\eta$'s. Since it is always possible to get 0-1 loss at most $\frac{1}{2}$ with a \textsc{Halfspace} predictor, we assume that $\eta_{++} + \eta_{--} \geq \eta_{+-} + \eta_{-+} = \mathcal{L}_{\mathcal{D}}^{01}(h^*)$. From the distribution $\mathcal{D}$ we construct a new distribution $\tilde{\mathcal{D}}$ over $(\tilde{X},\tilde{A},\tilde{Y})$ with $\tilde{X} \in \mathbb{R}^{d+1}$, $\tilde{A} \in \sset{0,1}$, and $\tilde{Y} \in \sset{-1,+1}$ in the following manner:
\begin{equation}
\begin{aligned}
&\mathbb{P}_{\tilde{\mathcal{D}}}(\tilde{A} = 0) &&= 1-\delta \\
&\mathbb{P}_{\tilde{\mathcal{D}}}(\tilde{A} = 1) &&= \delta \\
&\mathbb{P}_{\tilde{\mathcal{D}}}(\tilde{X} = -e_1, \tilde{Y} = -1\ |\ \tilde{A} = 0) &&= \eta_{--} \\
&\mathbb{P}_{\tilde{\mathcal{D}}}(\tilde{X} = -e_1, \tilde{Y} = +1\ |\ \tilde{A} = 0) &&= \eta_{-+} \\
&\mathbb{P}_{\tilde{\mathcal{D}}}(\tilde{X} = e_1, \tilde{Y} = -1\ |\ \tilde{A} = 0) &&= \eta_{+-} \\
&\mathbb{P}_{\tilde{\mathcal{D}}}(\tilde{X} = e_1, \tilde{Y} = +1\ |\ \tilde{A} = 0) &&= \eta_{++} \\
&\mathbb{P}_{\tilde{\mathcal{D}}}(\tilde{X} = [0, x], \tilde{Y} = y\ |\ \tilde{A} = 1) &&= \mathbb{P}_{\mathcal{D}}(X = x, Y = y)\qquad \forall x,y
\end{aligned}
\end{equation}
where $e_1$ is the first standard basis vector in $\mathbb{R}^{d+1}$. In other words, when $\tilde{A} = 1$ the distribution $\tilde{\mathcal{D}}$ is identical to $\mathcal{D}$ besides a zero appended to the beginning of $X$. When $\tilde{A} = 0$, $\tilde{\mathcal{D}}$ is supported on two points $-e_1$ and $e_1$. 

We will apply the algorithm $\mathcal{A}$ to the distribution $\tilde{\mathcal{D}}$ with parameters $\epsilon$ and $\alpha$ to be determined later. In this case $\mathcal{H}$ is the class of linear predictors so the hinge loss of $f$ on $\tilde{\mathcal{D}}$ must be competitve with the hinge loss of the best linear predictor whose sign is non-discriminatory.

Using the following lemmas, we show that the output of $\mathcal{A}$ must have small hinge loss on $\tilde{\mathcal{D}}$, that this output can then be modified so that it has small 0-1 loss on $\tilde{\mathcal{D}}$, and finally that it can be further modified to achieve small 0-1 loss on $\mathcal{D}$. The proofs are deferred to the end of this discussion.

\begin{lemma} \label{lem:hardnessexistsgoodhinge}
There exists a linear predictor $h$ whose sign is $0$-discrminatory such that 
\[ \mathcal{L}_{\tilde{\mathcal{D}}}^{\textrm{hinge}}(h) = 2(1-\delta)\mathcal{L}_{\mathcal{D}}^{01}(h^*) + 2\delta \]
\end{lemma}
\noindent By Lemma \ref{lem:hardnessexistsgoodhinge} and the defintion of $f$ from \eqref{eq:definitionAlg},
\begin{equation}
\mathcal{L}_{\tilde{\mathcal{D}}}^{\textrm{hinge}}\left(f\right) \leq \min_{h \in \mathcal{H}_{\textrm{0-disc}}(\tilde{\mathcal{D}})} \mathcal{L}_{\tilde{\mathcal{D}}}^{\textrm{hinge}}(h) + \epsilon \leq 2(1-\delta)\mathcal{L}_{\mathcal{D}}^{01}(h^*) + 2\delta + \epsilon
\end{equation}
and the sign of $f$ is $\alpha$-discriminatory. Next,
\begin{lemma} \label{lem:hardnessfixf}
The predictor $f$ can be efficiently modified to yield a new predictor $f'$ whose sign is
is $\alpha$-discriminatory such that 
\[ \mathcal{L}_{\tilde{\mathcal{D}}}^{01}(f') \leq (1-\delta)\mathcal{L}_{\mathcal{D}}^{01}(h^*) + 2\delta + \epsilon \]
\end{lemma}
Finally,
\begin{lemma} \label{lem:01lossessimilar}
The predictor $f''(x) = f'([0,x],1)$ achieves $\c{L}_{\c{D}}^{01}(f'') \leq \c{L}_{\tilde{\c{D}}}^{01}(f') + \alpha(1-\delta)$.
\end{lemma}
The predictor $f''$ described in Lemma \ref{lem:01lossessimilar} thus has 0-1 loss on $\mathcal{D}$
\begin{equation} \label{eq:hardness01loss}
\mathcal{L}_{\mathcal{D}}^{01}(f'') \leq (1-\delta)(\mathcal{L}_{\mathcal{D}}^{01}(h^*) + \alpha) + 2\delta + \epsilon
\end{equation}

Theorem 1.3 from \cite{daniely2015complexity} proves that there is no algorithm running in time polynomial in the dimension $d$ that can return a predictor achieving 0-1 error $\leq \frac{1}{2} - d^{-c}$ with high probability for a constant $c > 0$ for an arbitrary distribution, even with the knowledge that $\mathcal{L}_{\mathcal{D}}^{01}(h^*) \leq L^*$ for $L^* < 1/2$.
Thus, $\mathcal{A}(\tilde{D},\epsilon,\alpha)$ cannot run in time polynomial in the dimension $d$ for any parameters $\mathcal{L}_{\mathcal{D}}^{01}(h^*)$, $\epsilon$, $\alpha$, and $\delta$ such that \eqref{eq:hardness01loss} is greater than $\frac{1}{2} - d^{-c}$ for any $c > 0$. 
For $\epsilon,\alpha < \frac{1}{8}$, and setting $\delta = \frac{1}{16}$, \eqref{eq:hardness01loss} shows that
\begin{equation}
\mathcal{L}_{\mathcal{D}}^{01}(f'') \leq \frac{15}{16}\mathcal{L}_{\mathcal{D}}^{01}(h^*) + \frac{47}{128}
\end{equation}
For any $L^* < \frac{1}{10}$ this is at most $\frac{1}{2} - \frac{1}{40}$ and does not depend on the dimension. 

In this proof we assumed knowledge of the parameters $\eta$ which describe the conditional error rates of $h^*$. If a polynomial time algorithm for $\mathcal{A}$ existed, then it would be possible to perform two-dimensional grid search over the $\eta$. Calls made to $\mathcal{A}$ with the incorrect values of $\eta$ might result in very inaccurate or discriminatory predictors, but using an estimate of $\eta$ up to $\mathcal{O}(\alpha)$ accuracy is sufficient to approximate the \textsc{Halfspaces} solution using $\mathcal{A}$. Thus at most $\mathcal{O}(\log^2(1/\alpha))$ calls to the polynomial time algorithm would be needed. Therefore, in order for $\mathcal{A}$ to guarantee for an arbitrary $\tilde{\mathcal{D}}$ that its output would have excess hinge loss at most $\frac{1}{8}$ and its sign would be at most $\frac{1}{8}$-discriminatory, it must run in time super-polynomial in the dimension in the worst case.

\subsection{Deferred proofs}
\begin{proof}[Proof of Lemma \ref{lem:hardnessexistsgoodhinge}]
Define $h(X,A) = \textrm{sign}\left(\begin{bmatrix} 1 \\ w^* \end{bmatrix}^TX \right)$. Then
$h(-e_1,0) = -1$, $h(e_1,0) = 1$, and $h([0,x],1) = h^*(x)$, where $h^*$ is as defined in \eqref{eq:halfspacehstar}. Since the sign function is invariant to scaling, $\mathcal{L}_{\mathcal{D}}^{01}(h^*) = \mathcal{L}_{\mathcal{D}}^{01}(ch^*)$ for any $c > 0$. 

Theorem 1.3 in \cite{daniely2015complexity} involves a distribution $\mathcal{D}$ that is supported on the unit hypercube in $\mathbb{R}^d$, thus the predictor $\frac{h^*}{\|w^*\|_2\sqrt{d}} \in [-1,1]$ with probability 1, and has the same 0-1 loss as $h^*$.

By the definition of $\tilde{\mathcal{D}}$, $\eta_{+-}$, and $\eta_{-+}$:
\begin{equation}
\begin{aligned}
\mathbb{P}_{\tilde{\mathcal{D}}}(h \geq 0\ |\ \tilde{Y} = -1, \tilde{A} = 0 ) 
&= \frac{\eta_{+-}}{\mathbb{P}_{\tilde{\mathcal{D}}}(\tilde{Y} = -1\ |\ \tilde{A} = 0)} \\
&= \frac{\eta_{+-}}{\mathbb{P}_{\mathcal{D}}(Y = -1)} \\
&= \mathbb{P}_{\mathcal{D}}(h^* \geq 0\ |\ Y = -1) \\
&= \mathbb{P}_{\tilde{\mathcal{D}}}(h\phantom{^*} \geq 0\ |\ \tilde{Y} = -1, \tilde{A} = 1 )
\end{aligned}
\end{equation}
\begin{equation}
\begin{aligned}
\mathbb{P}_{\tilde{\mathcal{D}}}(h < 0\ |\ \tilde{Y} = 1, \tilde{A} = 0 ) 
&= \frac{\eta_{-+}}{\mathbb{P}_{\tilde{\mathcal{D}}}(\tilde{Y} = 1\ |\ \tilde{A} = 0)} \\
&= \frac{\eta_{-+}}{\mathbb{P}_{\mathcal{D}}(Y = 1)} \\
&= \mathbb{P}_{\mathcal{D}}(h^* < 0\ |\ Y = 1) \\
&= \mathbb{P}_{\tilde{\mathcal{D}}}(h\phantom{^*} < 0\ |\ \tilde{Y} = 1, \tilde{A} = 1 )
\end{aligned}
\end{equation}
Therefore, $h$ is $0$-discriminatory at threshold $0$. Also
\begin{align}
\mathcal{L}_{\tilde{\mathcal{D}}}^{\textrm{hinge}}(h)\
=\qquad &[1 + h(x_1)]_+\mathbb{P}_{\tilde{\mathcal{D}}}(\tilde{X} = x_1, \tilde{A} = 0, \tilde{Y} = -1) \\
+&[1 - h(x_0)]_+\mathbb{P}_{\tilde{\mathcal{D}}}(\tilde{X} = x_0, \tilde{A} = 0, \tilde{Y} = \phantom{-}1) \nonumber\\
+&\int_X[1 + h(x)]_+\mathbb{P}_{\tilde{\mathcal{D}}}(\tilde{X} = x, \tilde{A} = 1, \tilde{Y} = -1)dx \nonumber\\
+&\int_X[1 - h(x)]_+\mathbb{P}_{\tilde{\mathcal{D}}}(\tilde{X} = x, \tilde{A} = 1, \tilde{Y} = \phantom{-}1)dx \nonumber\\
\leq\qquad &2\eta_{+-}\mathbb{P}_{\tilde{\mathcal{D}}}(\tilde{A} = 0) + 2\eta_{-+}\mathbb{P}_{\tilde{\mathcal{D}}}(\tilde{A} = 0) \\
+& 2\int_X\left( \mathbb{P}_{\tilde{\mathcal{D}}}(\tilde{X} = x, \tilde{A} = 1, \tilde{Y} = -1) + \mathbb{P}_{\tilde{\mathcal{D}}}(\tilde{X} = x, \tilde{A} = 1, \tilde{Y} = 1) \right)dx \nonumber\\
=\qquad &2\mathbb{P}_{\tilde{\mathcal{D}}}(\tilde{A} = 0)(\eta_{+-} + \eta_{-+}) + 2\mathbb{P}_{\tilde{\mathcal{D}}}(\tilde{A} = 1) \\
=\qquad &2(1-\delta)(\eta_{-+} + \eta_{+-}) + 2\delta
\end{align}
\end{proof}

\begin{proof}[Proof of Lemma \ref{lem:hardnessfixf}]
Since only the sign of $f$ is required to be $\alpha$-discriminatory, we can modify the magnitude of its predictions without affecting its level of non-discrimination. Therefore, we first truncate the output of $f$ to lie in the range $[-1,1]$, which can only reduce the hinge loss. 

Ignoring for the moment that the sign of $f$ must be $\alpha$-discriminatory, we would like to define $f'$ so that $f'(-e_1,0) = -1$ and $f'(e_1,0) = 1$ with probability 1. In this case, the hinge loss when $\tilde{A} = 0$ is exactly $2(\eta_{+-} + \eta_{-+})$. With that being said, any modification to $f$ that changes the distribution of the \emph{sign} of the predictor risks rendering it more than $\alpha$-discriminatory. With this in mind, we construct $f'$ such that
\begin{equation}
\begin{aligned}
f'(-e_1,0) &= \begin{cases} -1 & \textrm{w.p. } \mathbb{P}(f(-e_1,0) < 0) \\ 0 & \textrm{w.p. } \mathbb{P}(f(-e_1,0) \geq 0) \end{cases}\\
f'(e_1,0) &= \begin{cases} 1 & \textrm{w.p. } \mathbb{P}(f(e_1,0) \geq 0) \\ -0 & \textrm{w.p. } \mathbb{P}(f(e_1,0) < 0) \end{cases} \\
f'([0,x],1) &= f([0,x],1)
\end{aligned}
\end{equation}
where $-0$ is a negative number of arbitrarily small magnitude. Constructed this way, the distribution of the sign of $f'$ conditioned on $\tilde{A}$ is identical to that of $f$, meaning that the sign of $f'$ is $\alpha$-discriminatory.

The hinge loss of $f'$ is an upper bound on the 0-1 loss, and in order to show that $f'$ achieves small 0-1 loss, we will show that the hinge loss is a \emph{loose} upper bound. The construction of $\tilde{\mathcal{D}}$ and the predictions of $f'$ conditioned on $\tilde{A} = 0$ creates a substantial gap between the losses.
\includegraphics[width=3in]{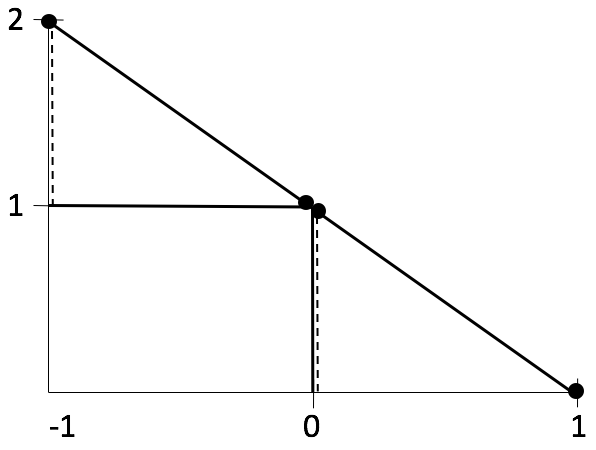}

Notice that when $f'$ makes a prediction of magnitude 1 that has the correct sign, both the hinge loss and the 0-1 loss evaluate to $0$. Similarly, when $f'$ makes a prediction of magnitude 0 with the incorrect sign, both losses are $1$. Thus in each of these cases, the hinge loss is equivalent to the 0-1 loss.

However, if $f'$ makes a prediction of magnitude 1 with the incorrect sign, the hinge loss is $2$ but the 0-1 loss is only $1$, and when $f'$ makes a prediction of magnitude 0 with the correct sign, the hinge loss is $1$ but the 0-1 loss is $0$. Consequently, in each of these cases there is a gap of $1$ between the hinge and 0-1 losses. Thus,
\begin{align} \label{eq:expectedgap}
\mathbb{E}\left[ \ell^{\textrm{hinge}}(f') - \ell^{01}(f') \, \middle|\, \tilde{A} = 0 \right] 
&= \mathbb{P}\left( |f'| = 1, \textrm{sign}(f') \neq \tilde{Y} \, \middle|\, \tilde{A} = 0 \right) \\
&+ \mathbb{P}\left( |f'| = 0, \textrm{sign}(f') = \tilde{Y} \, \middle|\, \tilde{A} = 0 \right) \nonumber
\end{align}
Considering each term separately:
\begin{align}
\mathbb{P}&\left( |f'| = 1, \textrm{sign}(f') \neq \tilde{Y} \ \middle|\ \tilde{A} = 0 \right) \nonumber\\
=&\ \mathbb{P}\left( f'(-e_1,0) = -1, \tilde{Y} = 1 \ \middle|\ \tilde{X} = -e_1, \tilde{A} = 0 \right)\mathbb{P}\left( \tilde{X} = -e_1 \ \middle|\ \tilde{A} = 0 \right) \\
&+ \mathbb{P}\left( f'(e_1,0) = 1, \tilde{Y} = -1 \ \middle|\ \tilde{X} = e_1, \tilde{A} = 0 \right)\mathbb{P}\left( \tilde{X} = e_1 \ \middle|\ \tilde{A} = 0 \right) \nonumber\\
=&\ \mathbb{P}\left( f'(-e_1,0) = -1 \right)\mathbb{P}\left( \tilde{X} = -e_1, \tilde{Y} = 1 \ \middle|\ \tilde{A} = 0 \right) \label{eq:predsindepY}\\
&+ \mathbb{P}\left( f'(e_1,0) = 1 \right)\mathbb{P}\left( \tilde{X} = e_1, \tilde{Y} = -1 \ \middle|\ \tilde{A} = 0 \right) \nonumber\\
=&\ \mathbb{P}\left( f(-e_1,0) < 0 \right)\eta_{-+}
+ \mathbb{P}\left( f(e_1,0) \geq 0 \right)\eta_{+-} \label{eq:expectedgap1}
\end{align}
With \eqref{eq:predsindepY} following from the fact that conditioned on $\tilde{X}$ and $\tilde{A} = 0$, $f'(\tilde{X},0)$ is a random variable that is independent of the value of $\tilde{Y}$. Similarly,
\begin{align}
&\mathbb{P}\left( |f'(\tilde{X},0)| = 0, \textrm{sign}(f'(\tilde{X},0)) = \tilde{Y} \ \middle|\ \tilde{A} = 0 \right) \nonumber\\
=&\ \mathbb{P}\left( f'(-e_1,0) = 0, \tilde{Y} = 1 \ \middle|\ \tilde{X} = -e_1, \tilde{A} = 0 \right)\mathbb{P}\left( \tilde{X} = -e_1 \ \middle|\ \tilde{A} = 0 \right) \\
&+ \mathbb{P}\left( f'(e_1,0) = -0, \tilde{Y} = -1 \ \middle|\ \tilde{X} = e_1, \tilde{A} = 0 \right)\mathbb{P}\left( \tilde{X} = e_1 \ \middle|\ \tilde{A} = 0 \right) \nonumber\\
=&\ \mathbb{P}\left( f'(-e_1,0) = 0 \right)\mathbb{P}\left( \tilde{X} = -e_1, \tilde{Y} = 1 \ \middle|\ \tilde{A} = 0 \right) \\
&+ \mathbb{P}\left( f'(e_1,0) = -0 \right)\mathbb{P}\left( \tilde{X} = e_1, \tilde{Y} = -1 \ \middle|\ \tilde{A} = 0 \right) \nonumber\\
=&\ \mathbb{P}\left( f(-e_1,0) \geq 0 \right)\eta_{-+} 
+ \mathbb{P}\left( f(e_1,0) < 0 \right)\eta_{+-}\label{eq:expectedgap2}
\end{align}
Combining \eqref{eq:expectedgap} with \eqref{eq:expectedgap1} and \eqref{eq:expectedgap2}, we see that
\begin{align}
\mathbb{E}\left[ \ell^{\textrm{hinge}}(f') - \ell^{01}(f') \ \middle|\ \tilde{A} = 0 \right] 
=&\ \mathbb{P}\left( f(-e_1,0) < 0 \right)\eta_{-+} + \mathbb{P}\left( f(e_1,0) \geq 0 \right)\eta_{+-} \\
&+ \mathbb{P}\left( f(-e_1,0) \geq 0 \right)\eta_{-+} + \mathbb{P}\left( f(e_1,0) < 0 \right)\eta_{+-} \nonumber\\
=&\ \eta_{-+} + \eta_{+-} \label{eq:thegap}
\end{align}
The 0-1 loss of $f'$ can be decomposed as
\begin{equation}
\mathcal{L}_{\tilde{\mathcal{D}}}^{\textrm{01}}(f') = \mathbb{P}(\tilde{A} = 0)\mathbb{E}\left[ \ell^{\textrm{01}}(f)\ \middle|\ \tilde{A} = 0 \right] + \mathbb{P}(\tilde{A} = 1)\mathbb{E}\left[ \ell^{\textrm{01}}(f)\ \middle|\ \tilde{A} = 1 \right]
\end{equation}
By \eqref{eq:thegap},
\begin{equation}
\mathbb{P}(\tilde{A} = 0)\mathbb{E}\left[ \ell^{\textrm{01}}(f)\ \middle|\ \tilde{A} = 0 \right] 
= (1- \delta)\left(\mathbb{E}\left[ \ell^{\textrm{hinge}}(f)\ \middle|\ \tilde{A} = 0 \right] - \eta_{-+} - \eta_{+-} \right)
\end{equation}
and since the hinge loss is always an upper bound on the 0-1 loss
\begin{equation}
\mathbb{P}(\tilde{A} = 1) \mathbb{E}\left[ \ell^{\textrm{01}}(f)\ \middle|\ \tilde{A} = 1 \right]
\leq \delta\mathbb{E}\left[ \ell^{\textrm{hinge}}(f)\ \middle|\ \tilde{A} = 1 \right]
\end{equation}
From Lemma \ref{lem:hardnessexistsgoodhinge}, the hinge loss of $f'$ (which is at most the hinge loss of $f$) is upper bounded by $2(1-\delta)(\eta_{-+} + \eta_{+-}) + 2\delta + \epsilon$. Thus,
\begin{align}
\mathcal{L}_{\tilde{\mathcal{D}}}^{\textrm{01}}(f') 
&\leq (1- \delta)\left(\mathbb{E}\left[ \ell^{\textrm{hinge}}(f)\ \middle|\ \tilde{A} = 0 \right] - \eta_{-+} - \eta_{+-} \right) 
+ \delta\mathbb{E}\left[ \ell^{\textrm{hinge}}(f)\ \middle|\ \tilde{A} = 1 \right] \\
&= \mathcal{L}_{\tilde{\mathcal{D}}}^{\textrm{hinge}}(f') - (1-\delta)(\eta_{-+} + \eta_{+-}) \\
&\leq (1-\delta)(\eta_{-+} + \eta_{+-}) + 2\delta + \epsilon
\end{align}
\end{proof}

\begin{proof}[Proof of Lemma \ref{lem:01lossessimilar}]
Because $f'$ is $\alpha$-discriminatory at threshold $0$
\begin{equation}
\begin{aligned}
\abs{\mathbb{P}_{\tilde{\c{D}}}(f' \geq 0\ |\ \tilde{Y} = -1, \tilde{A} = 0) - \mathbb{P}_{\tilde{\c{D}}}(f' \geq 0\ |\ \tilde{Y} = -1, \tilde{A} = 1)} &\leq \alpha \\
\abs{\mathbb{P}_{\tilde{\c{D}}}(f' < 0\ |\ \tilde{Y} = 1,\phantom{-} \tilde{A} = 0) - \mathbb{P}_{\tilde{\c{D}}}(f' < 0\ |\ \tilde{Y} = 1,\phantom{-} \tilde{A} = 1)} &\leq \alpha
\end{aligned}
\end{equation}
Let $f''(x) = f'([0,x], 1)$, then
\begin{align}
\mathcal{L}_{\tilde{\c{D}}}^{01}(h) 
=&\ \mathbb{P}_{\tilde{\c{D}}}(\tilde{Y} = -1, \tilde{A} = 0)\mathbb{P}_{\tilde{\c{D}}}(f' \geq 0\ |\ \tilde{Y} = -1, \tilde{A} = 0) \\
 &+\mathbb{P}_{\tilde{\c{D}}}(\tilde{Y} = 1,\phantom{-} \tilde{A} = 0)\mathbb{P}_{\tilde{\c{D}}}(f' < 0\ |\ \tilde{Y} = 1,\phantom{-} \tilde{A} = 0) \nonumber\\ 
&+\mathbb{P}_{\tilde{\c{D}}}(\tilde{Y} = -1, \tilde{A} = 1)\mathbb{P}_{\tilde{\c{D}}}(f' \geq 0\ |\ \tilde{Y} = -1, \tilde{A} = 1) \nonumber\\
 &+\mathbb{P}_{\tilde{\c{D}}}(\tilde{Y} = 1,\phantom{-} \tilde{A} = 1)\mathbb{P}_{\tilde{\c{D}}}(f' < 0\ |\ \tilde{Y} = 1,\phantom{-} \tilde{A} = 1) \nonumber\\
\geq&\ \mathbb{P}_{\tilde{\c{D}}}(\tilde{Y} = -1, \tilde{A} = 0)\left(\mathbb{P}_{\tilde{\c{D}}}(f' \geq 0\ |\ \tilde{Y} = -1, \tilde{A} = 1) - \alpha \right) \\
&+\mathbb{P}_{\tilde{\c{D}}}(\tilde{Y} = 1,\phantom{-} \tilde{A} = 0)\left( \mathbb{P}_{\tilde{\c{D}}}(f' < 0\ |\ \tilde{Y} = 1,\phantom{-} \tilde{A} = 1) - \alpha \right) \nonumber\\ 
&+\mathbb{P}_{\tilde{\c{D}}}(\tilde{Y} = -1, \tilde{A} = 1)\mathbb{P}_{\tilde{\c{D}}}(f' \geq 0\ |\ \tilde{Y} = -1, \tilde{A} = 1) \nonumber\\
&+\mathbb{P}_{\tilde{\c{D}}}(\tilde{Y} = 1,\phantom{-} \tilde{A} = 1)\mathbb{P}_{\tilde{\c{D}}}(f' < 0\ |\ \tilde{Y} = 1,\phantom{-} \tilde{A} = 1) \nonumber\\
=&\ \left(\mathbb{P}_{\tilde{\c{D}}}(\tilde{Y} = -1, \tilde{A} = 0) + \mathbb{P}_{\tilde{\c{D}}}(\tilde{Y} = -1, \tilde{A} = 1)\right) \mathbb{P}_{\tilde{\c{D}}}(f' \geq 0\ |\ \tilde{Y} = -1, \tilde{A} = 1) \\
&+\left(\mathbb{P}_{\tilde{\c{D}}}(\tilde{Y} = 1,\phantom{-} \tilde{A} = 0) + \mathbb{P}_{\tilde{\c{D}}}(\tilde{Y} = 1,\phantom{-} \tilde{A} = 1) \right)\mathbb{P}_{\tilde{\c{D}}}(f' < 0\ |\ \tilde{Y} = 1,\phantom{-} \tilde{A} = 1) \nonumber\\
&+\left(\mathbb{P}_{\tilde{\c{D}}}(\tilde{Y} = -1, \tilde{A} = 0) + \mathbb{P}_{\tilde{\c{D}}}(\tilde{Y} = 1,\phantom{-} \tilde{A} = 0) \right)(-\alpha) \nonumber\\
=&\ \mathbb{P}_{\tilde{\c{D}}}(\tilde{Y} = -1) \mathbb{P}_{\tilde{\c{D}}}(f' \geq 0\ |\ \tilde{Y} = -1, \tilde{A} = 1) \\
& + \mathbb{P}_{\tilde{\c{D}}}(\tilde{Y} = 1)\phantom{-}\mathbb{P}_{\tilde{\c{D}}}(f' < 0\ |\ \tilde{Y} = 1,\phantom{-} \tilde{A} = 1) - \alpha\mathbb{P}_{\tilde{\c{D}}}(\tilde{A} = 0) \\
=&\ \mathbb{P}_{\c{D}}(Y = -1) \mathbb{P}_{\c{D}}(f'' \geq 0\ |\ Y = -1) \\
&+ \mathbb{P}_{\c{D}}(Y = 1)\mathbb{P}_{\c{D}}(f'' < 0\ |\ Y = 1) - \alpha(1-\delta) \nonumber\\
=&\ \c{L}_{\c{D}}^{01}(f'') - \alpha(1-\delta)
\end{align}
The lemma follows immediately.
\end{proof}

\section{Proofs for Section \ref{sec:2ndorder} - Relaxing non-discrimination} \label{appendix:2ndorder}

We use $\Cov{\cdot}$ to denote covariances involving a vector and we reserve $\cov{\cdot}$ for scalar covariances. We start with recalling some facts about Gaussian random variables.

\begin{proposition}
If $(U,V,W)$ are jointly Gaussian, then:
\begin{itemize}
 \item \textbf{Conditional expectation} $\E[U|V]$ is linear in $V$ and is given by:
    \[
        \E[U|V] = \E[U] + \Cov{U,V} \Cov{V}^{-1} (V-\E[V]).
    \]
 \item \textbf{Conditional covariance} $\Cov{U,V|W}$ does not depend on $W$ and it is always equal to:
    \[
        \Cov{U,V|W} = \Cov{U,V}-\Cov{U,W}\Cov{W}^{-1}\Cov{W,V}.
    \]
\end{itemize}
\end{proposition}

\subsection{Proof of Theorem \ref{thm:gaussian}}

First, let us show the second claim. Let $R$ be any linear predictor. By linearity, it follows that $(R,A,Y)$ are jointly Gaussian. By the conditional covariance formula, we have:
\[
    \cov{R A|Y}=\cov{R A} - \cov{R Y}\cov{Y A}/\var{Y}.
\]
If $R$ satisfies equalized correlations, then the right-hand side here is $0$. It follows that $R$ and $A$ are uncorrelated conditionally on $Y$. But since they are jointly Gaussian, they are also independent conditionally on $Y$. Therefore $R$ also satisfies equalized odds non-discrimination. The converse also holds: if $R$ satisfies equalized odds, then $R$ and $A$ are uncorrelated given $Y$, and therefore equalized correlations is satified.

Now let us move back to the main claim. Assume, without loss of generality, that all variables are centered. Let us first find the optimal (a priori not necessarily linear) predictor that satisfies the relaxed second-moment non-discrimination criterion. In particular, the Lagrangian to minimize may be written as:
\[
    \E[(R-Y)^2]-\lambda\left(\var{Y}\E[RA]-\cov{AY}\E[RY]\right),
\]
But just like in the unconstrained least squares problem, we may apply the law of total expectatons to condition the loss and the $R$- terms in the second-moment non-discrimination constraint to be conditioned on $X$ and $A$. Thus the optimum is achieved for each $X,A$ by minimizing the following Lagrangian:
\[
    \E[(R-Y)^2|X,A]-\lambda\left(\var{Y}AR-\cov{AY}\E[Y|X,A]R\right),
\]
or equivalently
\[
    R^2-2\E[Y|X,A]R-\lambda\left(\var{Y}AR-\cov{AY}\E[Y|X,A]R\right).
\]
It follows that the optimal $R$ is a linear function of $A$, $\E[Y|X,A]$ and $\lambda$. $\lambda$ is determined over the statistics of the problem, and therefore it is a constant that does not depend on specific values of $X$ and $A$, and $\E[Y|X,A]$ in the Gaussian setting is linear. It thus follows that the optimum, let's call it $R_\circ$, is a linear function of $X$ and $A$. It minimizes the expected square loss subject to a relaxed non-discrimination criterion, therefore it is not larger than the optimizer under the stricter constraint. Conversely, by linearity it does also satisfy the stricter constraint, and is thus no smaller than the optimizer under that constraint (recall that we didn't start out by imposing linearity). Therefore $R_\circ$ is precisely the squared loss optimal equalized odds non-discriminatory predictor. 

\subsection{Optimal equalized correlations linear predictor} \label{proof:opt-eq-cor-lin}

We write the proofs more generally for a vector-valued protected attribute $A$, and the scalar case follows directly. In this case $\uv$ is a matrix and $\alpha$ is a vector mixing the columns of $\uv$ (so the correction has the term $v\alpha$). First note that condition \eqref{eq:equalized-correlations} translates into a linear constraint on $\weight$ in the least-squares problem of \eqref{eq:opt-eq-cor-lin}. Using the bilinearity of the covariance, this constraint is equivalent to:
\[
    \weight\tp \Cov{\XA,A}-\weight\tp\Cov{\XA,Y}~\Cov{Y,A}/\var{Y}\equiv\weight\tp\uv=0.
\]
We can now write the cost function with a vector of Lagrange multipliers:
\[
    J(\weight,\lambda)=\E[(Y-\weight\tp \XA)^2]+ \weight\tp \uv\lambda,
\]
whose gradient is
\[
    \nabla_\weight J = 2\Cov{\XA,\XA}\weight -2\Cov{\XA,Y}+\uv\lambda.
\]
Setting this to zero, the optimality conditions give us the claimed functional form. The vector $\alpha$ can then be obtained by enforcing the constraint:
\[
    \weight\tp \uv = \left(\Cov{Y,\XA} -\alpha\tp~\uv\tp \right)\Cov{\XA,\XA}^{-1}\uv=0
\]
and thus
\[
  \uv\tp\Cov{\XA,\XA}^{-1}\uv\alpha=\uv\tp\Cov{\XA,\XA}^{-1}\Cov{\XA,Y}.
\]

\subsection{Proof of Theorem \ref{thm:no-gap}}
First, let us rewrite $R^\star$ as:
\begin{eqnarray*}
    R^\star
        &=&{\weight^\star}\tp\XA\\
        &=&{\widehat\weight}\tp\XA - \alpha\tp \uv\tp\Cov{\XA,\XA}^{-1} \XA\\
        &=&\widehat R - \alpha\tp \uv\tp\Cov{\XA,\XA}^{-1} \XA\\
\end{eqnarray*}
Next, recall that
\[
    \uv = \Cov{\XA,A} - \Cov{\XA,Y}\Cov{Y,A}/\var{Y},
\]
and since
\[
    \Cov{\XA,\XA} = \smatrix{\Cov{X,\XA} \\ \Cov{A,\XA}}\ \Rightarrow\ \Cov{A,\XA} \Cov{\XA,\XA}^{-1} = \smatrix{ 0_{\dim(A)\times\dim(X)} & I_{\dim(A)}},
\]
we have
\[
    \uv \tp \Cov{\XA,\XA}^{-1}= \smatrix{0 & I} - \Cov{A,Y}{\widehat\weight}\tp/\var{Y}.
\]
Therefore 
\[
    \uv\tp\Cov{\XA,\XA}^{-1}\XA= A-\frac{\widehat R}{\var{Y}}\Cov{A,Y}
\]
and can be derived from $(\widehat R,A,Y)$. Then multiplying from the right by $\uv$ and using the bilinearity of the covariance, we get the terms in $\alpha$:
\[
    \uv \tp \Cov{\XA,\XA}^{-1}\uv=\Cov{A,A} -\tfrac{1}{\var{Y}}\Cov{A,Y}\Cov{Y,A}-   \tfrac{1}{\var{Y}}\Cov{A,Y}\left(\Cov{\widehat R,A}-\tfrac{1}{\var{Y}}\Cov{\widehat R,Y}\Cov{Y,A}\right),
\]
\[
    \textrm{and}\  \uv \tp \Cov{\XA,\XA}^{-1}\Cov{\XA,Y}=\Cov{A,Y} - \tfrac{1}{\var{Y}}\Cov{A,Y}\Cov{\widehat R,Y}.
\]
This shows that $\alpha$ also derives from $(\widehat R,A,Y)$ as stated. To simplify the expression of $\alpha$ to the one claimed, note that we have $\Cov{\widehat R,A}=\Cov{Y,A}$. This is because $\widehat R$ is the squared loss optimal linear predictor of $Y$ given $A$ and $X$, and thus $\widehat R-Y$ is uncorrelated with any linear function of $A$ and $X$, and in particular $A$. This completes the proof.

\end{document}